\documentclass{article}

\usepackage{microtype}
\usepackage{graphicx}
\usepackage{subfigure}
\usepackage{booktabs} 

\usepackage{hyperref}

\usepackage[accepted]{icml2024}

\usepackage{titletoc}
\usepackage[page, header, toc, page]{appendix} 

\definecolor{light-gray}{gray}{0.9}
\definecolor{darkblue}{rgb}{0.0,0.0,0.65}
\definecolor{darkred}{rgb}{0.2,0.0,0.0}
\hypersetup{
	colorlinks = true,
	citecolor  = darkblue,
	linkcolor  = darkred,
	filecolor  = darkblue,
	urlcolor   = darkblue,
}
\usepackage{url}            
\usepackage{booktabs}       
\usepackage{amsfonts}       
\usepackage{nicefrac}       
\usepackage{microtype}      
  
\usepackage{amssymb}
 
\usepackage{footnote}
\makesavenoteenv{algorithm}
 
\usepackage{url}            %
\usepackage{nicefrac}       %
\usepackage{amsfonts, amsmath, amssymb, amsthm, bm, mathtools,dsfont}

\mathtoolsset{showonlyrefs}

\usepackage{mathrsfs}

\usepackage{enumitem}
\usepackage{natbib}
\usepackage{thmtools} 
\usepackage{thm-restate}
\usepackage{wrapfig}
\usepackage{nicefrac}
\usepackage{xfrac}
\usepackage{mdframed}
\usepackage[disable]{todonotes}

\usepackage{multirow}
\usepackage{booktabs}
\usepackage{makecell}

\usepackage{hhline}

\setlist{nosep,leftmargin=0.2in}

\DeclareMathOperator*{\argmin}{arg\,min}

\newcommand{\inp}[2]{\left\langle #1,#2 \right\rangle}
\newcommand{\R}{\mathbb{R}}
\newcommand*{\E}{\mathbb{E}}

\newcommand{\eps}{\epsilon}
\newcommand{\be}[1]{L_{{ \tiny #1}}}  
 \newcommand{\vv}{{\bm v}}
\newcommand{\x}{{\bm x}}
\newcommand{\per}{{\bm g}}
\newcommand{\ttr}{\overline{\mathsf{tr}}}
\newcommand{\Xstar}{\mathcal{X}^\star}
 
\newcommand{\partr}{{\bm \nabla}_t}
\newcommand{\proj}{\mathrm{Proj}}
\newcommand{\D}{\mathrm{d}}\newcommand{\loss}{ f}

\newcommand{\mm}{p}
\newcommand{\pp}{{\bm u}}
\newcommand{\consta}{L_0}
\newcommand{\constb}{L_1}
\newcommand{\constc}{L_2}

\newcommand{\xh}{\widehat{\x}}

\newcommand{\oo}[1]{\mathcal{O}\left( #1\right)}
\newcommand{\ooo}[1]{{o}\left( #1\right)}
\newcommand{\tth}[1]{{\Theta}\left( #1\right)}
\newcommand{\om}[1]{ {\Omega}\left(#1\right)}

\DeclareMathOperator{\tr}{tr} 
\newcommand{\norm}[1]{\left\|#1\right\|}

\newtheorem{theorem}{Theorem} 
\newtheorem{definition}{Definition}
\newtheorem{lemma}{Lemma}

\newtheorem{remark}{Remark}

\newtheorem{setting}{Setting}
\newtheorem{example}{Example}

\icmltitlerunning{How to Escape Sharp Minima}

\begin{document}

\twocolumn[
\icmltitle{How to Escape Sharp Minima with Random Perturbations}




\begin{icmlauthorlist}
\icmlauthor{Kwangjun Ahn}{mit,msr}
\icmlauthor{Ali Jadbabaie}{mit}
\icmlauthor{Suvrit Sra}{mit,tu} 
\end{icmlauthorlist}

\icmlaffiliation{mit}{MIT}
\icmlaffiliation{tu}{TU Munich}
\icmlaffiliation{msr}{Microsoft Research}

\icmlcorrespondingauthor{Kwangjun Ahn}{kjahn@mit.edu} 

\icmlkeywords{Machine Learning, ICML}

\vskip 0.3in
]



\printAffiliationsAndNotice{}  

\begin{abstract}
Modern machine learning applications have witnessed the remarkable success of optimization algorithms that are designed to find flat minima. Motivated by this design choice, we undertake a formal study that  (i) formulates the notion of flat minima, and (ii) studies the complexity of finding them. Specifically, we adopt the trace of the Hessian of the cost function as a measure of flatness, and use it to formally define the notion of approximate flat minima. 
Under this notion, we then analyze algorithms that find approximate flat minima efficiently. 
For general cost functions, we discuss a gradient-based algorithm that finds an approximate flat local minimum efficiently. The main component of the algorithm is to use gradients computed from randomly perturbed iterates to estimate a direction that leads to flatter minima.
For the setting where the cost function is an empirical risk over training data, we present a faster algorithm that is inspired by a recently proposed practical algorithm called sharpness-aware minimization, supporting its success in practice.
\end{abstract}

\section{Introduction}\label{sec:intro}

In modern machine learning applications, the training loss function $f:\R^d \to \R$ to be optimized often has a continuum of local/global minima, and the central question is which minima lead to good prediction performance. 
Among many different properties for minima, ``flatness'' of minima has been a promising candidate extensively studied in the literature~\citep{hochreiter1997flat,keskar2017large,dinh2017sharp, dziugaite2017computing,neyshabur2017exploring,sagun2017empirical,yao2018hessian,chaudhari2019entropy,  he2019asymmetric, mulayoff2020unique,tsuzuku2020normalized,xie2021diffusion}.

Recently, there has been a resurgence of interest in flat minima due to various advances in  both empirical and theoretical domains. Motivated by the extensive research on flat minima, this work undertakes a formal study that
\begin{itemize}
\item[\emph{(i)}] delineates a clear definition for flat minima, and 
\item[\emph{(ii)}] studies the upper complexity bounds of finding them.
\end{itemize}
We begin by emphasizing the significance of flat minima, based on recent advancements in the field.
\subsection{Why Flat Minima?}
\label{sec:flat}

Several recent optimization methods that are explicitly designed to find flat minima  have achieved substantial empirical success \citep{chaudhari2019entropy,izmailov2018averaging, foret2020sharpness, wu2020adversarial, zheng2021regularizing,norton2021diametrical,kaddour2022flat}.
One notable example is sharpness-aware minimization (SAM)~\citep{foret2020sharpness}, which has shown significant improvements in prediction performance of deep neural network models for image classification problems \citep{foret2020sharpness}  and language processing problems~\citep{bahri2022sharpness}. 
Furthermore, research by \citet{liu2023same} indicates that for language model pretraining, the flatness of minima serves as a more reliable predictor of model efficacy than the pretraining loss itself, particularly when the loss approaches its minimum values.

Complementing the empirical evidence, recent theoretical research underscores the importance of flat minima as a desirable attribute for optimization. Key insights include:
\begin{itemize}
\item \emph{Provable generalization of flat minima.} For overparameterized models, research by \citet{ding2022flat} demonstrates that flat minima correspond to the true solutions in low-rank matrix recovery tasks, such as matrix/bilinear sensing, robust PCA, matrix completion, and regression with a single hidden layer neural network, leading to better  generalization. 
This is further extended by \citet{gatmiry2023what} to deep linear networks learned from linear measurements. 
In other words, in a range of nonconvex problems with multiple minima,  flat minima yield superior predictive performance.  

\item \emph{Benifits of flat minima in pretraining.} Along with the empirical validations, \citet{liu2023same} prove that in simplified masked language models, flat minima correlate with the most generalizable solutions.

\item  \emph{Inductive bias of algorithms towards flat minima.} It has been proved that various practical optimization algorithms inherently favor flat minima. 
This includes stochastic gradient descent (SGD)~\citep{blanc2020implicit,wang2022large,damian2021label, li2022what, liu2023same}, gradient descent (GD) with large learning rates~\cite{arora2022understanding,damian2022self,ahn2023learning}, sharpness-aware minimization (SAM)~\citep{bartlett2022dynamics,wen2022does,compagnoni2023sde,dai2023the}, and a communication-efficient variant of SGD~\citep{gu2023why}. 
The practical success of these algorithms indicates that flatter minima might be linked to better generalization properties.
\end{itemize}

Motivated by such recent advances, the main goal of this work is to initiate a formal study of the behavior of algorithms for finding flat minima, especially an understanding of their upper complexity bounds.

\subsection{Overview of Our Main Results}

In this work, we formulate a particular notion of flatness for minima and design efficient algorithms for finding them.
We adopt the trace of Hessian of the loss $\tr(\nabla^2 f(\x))$ as a measure of ``flatness,'' where lower values of the trace imply flatter regions within the loss landscape. The reasons governing this choice are many, especially its deep relevance across a rich variety of research, as summarized in \autoref{sec:trace}. With this metric, we characterize a flat minimum as a local minimum where any local enhancement in flatness would result in an increased cost, effectively delineating regions where the model is both stable and efficient in terms of performance. 
More formally, we define the notion of $(\eps,\eps')$-flat minima in \autoref{def:flat_min}. 
See \autoref{sec:define}   for precise details.

Given the notion of flat minima, the main goal of this work is to design algorithms that find an approximate flat minimum efficiently.  
At first glance, the goal of finding a flat minimum might seem computationally expensive because minimizing $\tr(\nabla^2 f)$ would require information about second or higher derivatives.
Notably,  this work demonstrates that one can reach a flat minimum  using only first derivatives (gradients).
\begin{itemize}
\item In \autoref{sec:RS}, we present a gradient-based algorithm called the randomly smoothed perturbation algorithm (\autoref{algo:RS}) which  finds a $(\eps,\sqrt{\eps})$-flat minimum within $\oo{\eps^{-3}}$ iterations for general costs without structure (\autoref{thm:RS}). The main component of the algorithm is to use gradients computed from randomly perturbed iterates to  estimate a direction that leads to flatter minima.

\item  In \autoref{sec:SAM}, we consider the setting where  $f$ is the training loss over a training data set and the initialization is near the set of global minima, motivated by  overparametrized models in practice~\citep{zhang2021understanding}.
In such a setting, we present another gradient-based algorithm called the sharpness-aware perturbation algorithm (\autoref{algo:SAM}), inspired by sharpness-aware minimization (SAM) \citep{foret2020sharpness}.
We show that this algorithm  finds a $(\eps,\sqrt{\eps})$-flat minimum within $\oo{d^{-1}\epsilon^{-2}(1\vee\frac{1}{d^3\eps})}$ iterations (\autoref{thm:SAM}) -- here $d$ denotes the dimension of the domain. This demonstrates that a practical algorithm like SAM can find flat minima much faster than the randomly smoothed perturabtion algorithm in  high dimensional settings. 

\end{itemize}

See \autoref{tab:results} for a high level summary of our results.

\renewcommand{\arraystretch}{1.5}
\begin{table*}[t]
\centering
\begin{tabular}
{|>{\centering\arraybackslash}m{1.4in} |>{\centering\arraybackslash}m{2.2in} |>{\centering\arraybackslash}m{2in} | }	
\hline 
{Setting}  &  {Iterations for $(\eps,\sqrt{\eps})$-flat minima (\autoref{def:flat_min})}&  {Algorithm}  \\
\hhline{|===|}
\multirow{1}{1in}{\centering General loss} 
&   $\bm{\oo{\epsilon^{-3}}}$ {\bf gradient queries} (\autoref{thm:RS}) &   {\bf Randomly Smoothed Perturbation} (\autoref{algo:RS})   \\
\hhline{|===|}
\multirow{1}{1in}{\centering Training loss (\autoref{def:ER})}  
&   \makecell{$\bm{ \oo{d^{-1}\epsilon^{-2}(1\vee\frac{1}{d^3\eps})}}$ \\ {\bf gradient queries} (\autoref{thm:SAM})} & {\bf Sharpness-Aware Perturbation}  (\autoref{algo:SAM})    \\
\hline
\end{tabular}
\caption{A high level summary of the main results with emphasis on the dependence on $d$ and $\epsilon$.   
}
\label{tab:results}
\end{table*}

\section{Formulating Flat Minima}
\label{sec:define}

In this section, we formally define the notion of flat minima.

\subsection{Measure of Flatness} 
\label{sec:trace}

Within  the literature reviewed in \autoref{sec:flat},  a recurring metric for evaluating  ``flatness'' in loss landscapes is the trace of the Hessian matrix of the loss function $\tr(\nabla^2 f(\x))$. 
This metric intuitively reflects the curvature of the loss landscape around minima, where the Hessian matrix $\nabla^2 f(\x)$ is expected to be positive semi-definite. Consequently, lower values of the trace indicate regions where the loss landscape is flatter.
For simplicity, we will refer to this metric as \emph{the trace of Hessian}. 
Key insights from recent research include:

\begin{figure} 
\centering
\centering
\includegraphics[width=0.5\textwidth]{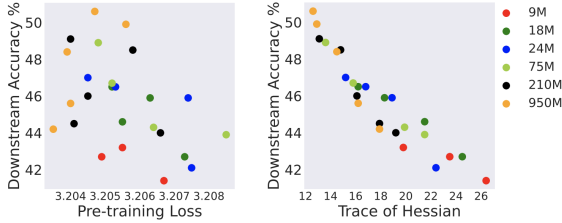}
\caption{\footnotesize{Figure from \citet{liu2023same}. They pretrain language models for probabilistic context-free grammar with different optimization methods, and compare their downstream accuracy. As shown in the plot, the trace of Hessian is a better indicator of the performance than the pretraining loss itself.}}
\label{fig:liu}
\end{figure}

\begin{figure} 
\centering
\includegraphics[width=0.3\textwidth]{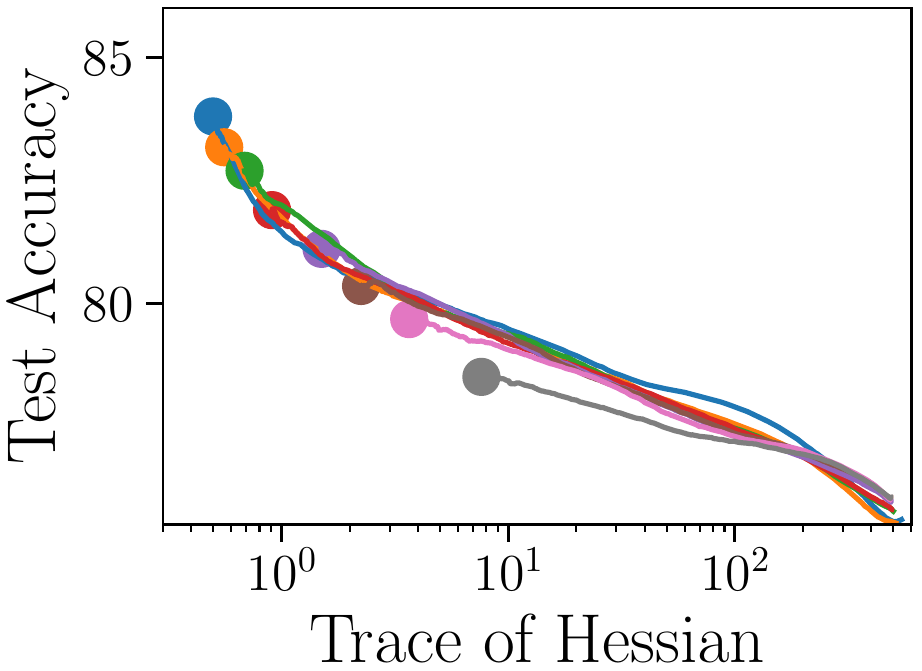}
\caption{\footnotesize{Figure from \citet{damian2021label}. For training ResNet-18 on CIFAR 10, they measure the trace of Hessian across the iterates of SGD with label noise and have observed an inspiring relation between $\tr(\nabla^2 f(\x_t))$ and prediction performance.}}
\label{fig:damian}
\end{figure}

\begin{itemize}
\item \emph{Overparameterized low-rank matrix recovery.} In this domain, the trace of Hessian has been identified as the \emph{correct notion of  flatness}.
\citet{ding2022flat} show that the most desirable minima, which correspond to ground truth solutions, are those with the lowest trace of Hessian values.
This principle is also applicable to the analysis of deep linear networks, as highlighted by \citet{gatmiry2023what}, where the same measure plays a pivotal role.
\item \emph{Language model pretraining.}
The importance of the trace of Hessian extends to language model pretraining, as demonstrated by \citet{liu2023same}.
More specifically, \citet{liu2023same} conduct an insightful experiment (see \autoref{fig:liu}) that demonstrates the effectiveness of the trace of Hessian as a good measure of model performance.
This observation is backed by their theoretical results for simple language models.

\item \emph{Model output stability.} Furthermore,  the work~\citep{ma2021linear} links the trace of Hessian to the stability of model outputs in deep neural networks relative to input data variations. This relationship underscores the significance of the trace of Hessian in improving model generalization and enhancing adversarial robustness.

\item \emph{Practical optimization algorithms.} Lastly, various practical optimization algorithms are shown to be inherently biased toward achieving lower values of the trace of Hessian. This includes SGD with label noise, as discussed in works by \citet{blanc2020implicit,damian2021label,li2022what}, and without label noise for the language modeling pretraining \citep{liu2023same}. 
In particular, \citet{damian2021label} conduct an inspiring experiment showing a strong correlation between the trace of Hessian and the prediction performance of models (see \autoref{fig:damian}).
Additionally, stochastic SAM is proven to prefer lower trace of Hessian values~\citep{wen2022does}.
\end{itemize}

\begin{remark}[{\bf Other notions of flatness?}] 
\label{rmk:other}
Perhaps, another popular notion of flat minima in the literature is the maximum eigenvalue of Hessian $\lambda_{\max}(\nabla^2 f(\x))$. However, recent empirical works~\citep{kaur2023maximum,andriushchenko2023modern}  have shown that the maximum eigenvalue of Hessian has limited correlation with the goodness of models (\emph{e.g.}, generalization). 
On the other hand, as we detailed above, the trace of Hessian has been consistently brought up as a promising candidate, both {\bf theoretically and empirically}. Hence, we adopt the trace of Hessian as the measure of flatness throughout.

\end{remark}

\subsection{Formal Definition of Flat Minima}
Motivated by the previous works discussed above, we  adopt the trace of Hessian as the measure of flatness.
Specifically, we consider the (normalized) trace of Hessian $\tr(\nabla^2 f(\x))/\tr(I_d) =\tr(\nabla^2 f(\x))/d$. Here we use the normalization to match the scale of flatness with the loss. For simplicity, we henceforth use the following notation: 
\begin{align}
\boxed{\ttr(x) \coloneqq  \frac{\tr(\nabla^2 f(\x))}{\tr(I_d)} = \frac{ \tr(\nabla^2 f(\x))}{d}\,.}
\end{align} 
The reason we consider the normalized trace is to match its scale with that of loss $f(\x)$: the trace is in general the sum of $d$ second derivatives, so it's scale is $d$ times of that of $f(\x)$. Also, the normalization can be potentially beneficial in practice where models have different sizes. Larger models would typically have a higher trace of Hessian due to having more parameters, and the normalization could put them on the same scale.

Given this choice, our notion of flat minima at a high level is \emph{a local minimum (of $f$) for which one cannot locally decrease $\ttr$  without increasing the cost $f$.} In particular, this concept becomes nontrivial when the set of local minima is connected (or locally forms a manifold), which is indeed the case for the over-parametrized neural networks, as  shown empirically in \citep{draxler2018essentially,garipov2018loss} and theoretically in \citep{cooper2021global}.

One straightforward way to define a (locally) flat minimum is the following: a local minimum which is also a local minimum of $\ttr$.
However, this definition is not well-defined as the set of local minima of $f$ can be disjoint from that of $\ttr$ as shown in the following example.
\begin{example}
Consider a  two-dimensional  function $f: (x_1,x_2) \mapsto (x_1x_2-1)^2$. 
Then it holds that 
\begin{align}
\nabla f(\x)
&=  {\footnotesize 2(x_1x_2-1)  \begin{bmatrix} x_2 \\ x_1 \end{bmatrix}}\quad \text{and}\\
\nabla^2 f(\x)
&=  {\footnotesize \begin{bmatrix}  2x_2^2 & 4x_1x_2-2\\
4x_1x_2- 2 & 2x_1^2\end{bmatrix}}\,.
\end{align}
Hence, the set of minima is $\Xstar= \{ (x_1,x_2)~:~ x_1x_2=1\}$ and $\ttr(\x) = (2x_1^2 +2x_2^2)/2 =x_1^2 + x_2^2$.
The unique minimum of $\ttr$ is $(0,0)$ which does not intersect with  $\Xstar$.
When restricted to $\Xstar$,  $\ttr$ achieve its minimum at $(1,1)$ and $(-1,-1)$, so those two points are flat minima.
\end{example}
Hence,  we consider the local optimality of $\ttr$ restricted to the set of local minima $\Xstar$.  
In practice, finding local minima with respect to  $\ttr$ might be too stringent, so as an initial effort,  we set our goal to find \emph{a local minimum that is also a stationary point of  $\ttr$ restricted to the set of local minima}. 
To formalize this, we introduce the limit map under the gradient flow, following \citep{li2022what,arora2022understanding,wen2022does}.
\begin{definition}[Limit point under gradient flow] \label{def:Phi}
Given a point $\x$, let $\Phi(\x)$ be the limiting point of the gradient flow on $f$ starting at $\x$. 
More formally,  letting $\x(t)$ be the iterate at time $t$ of the gradient flow starting at $\x$, i.e., $\x(0)=\x$ and $\dot \x(t) = -\nabla f(\x(t))$,   $\Phi(\x)$ is defined as $\lim_{t \to \infty} \x(t)$. 
\end{definition}

The intuition behind such a definition is the following. Since we are focusing on the first-order optimization algorithms that has access to gradients of $f$, the natural notion of optimality is the local optimality. In other words, we want to ensure that at a flat minimum, locally deviating away from the minimum will either increase the loss or the trace of Hessian. This condition precisely corresponds to $[\nabla(\overline{tr}\circ \Phi)](x^\star) =0$, since $\Phi$ maps each point $x$ to its "closest" local minimum.

When $\x$ is near a set of local minima, $\Phi(\x)$ is approximately equal to the projection onto the local minima set.
Thus, the trace of Hessian $\ttr$ along the manifold can be captured by the functional $\ttr(\Phi(x))$.
Therefore, we say a local minimum $\x^\star$ is a stationary point of $\ttr$ restricted to $\Xstar$ if
\begin{align}
\nabla \left[\ttr( \Phi(\x^\star)) \right]  = \partial \Phi(\x^\star) \nabla  \ttr (\Phi(\x^\star))  = \mathbf{0}\,.
\end{align}
In particular, if $ \nabla \left[\ttr( \Phi(\x^\star)) \right]\neq \mathbf{0}$, moving along the direction of $ -\nabla \left[\ttr( \Phi(\x^\star)) \right]$ will locally decrease the value $\ttr(\Phi(\x))$ while staying within the set of minima, hence leading to a flatter minimum. 
Moreover,  if $\x^\star$ is  an isolated local minimum, then $\partial \Phi(\x^\star) =\mathbf{0}$, and hence  $ \nabla \left[\ttr( \Phi(\x^\star))) \right]  = \mathbf{0}$.
This leads to the following definition. 

\begin{definition}[{\bf Flat local minima}]
\label{def:flat_exact}
We say a point $\x$ is a flat local minimum if it is a local minimum, \emph{i.e.}, $\x\in\Xstar$, and satisfies 
\begin{align} \label{eq:flatness_condition}
\nabla \left[\ttr( \Phi(\x)) \right]  = \partial \Phi(\x) \nabla  \ttr (\Phi(\x))  = \mathbf{0}\,. \end{align} 
\end{definition}

Again, the intuition behind \autoref{def:flat_exact} is that we want to ensure that at a flat minimum, locally deviating away from the minimum will either increase the loss or the trace of Hessian. This condition precisely corresponds to  \eqref{eq:flatness_condition}, since $\Phi$ maps each point $x$ to its "closest" local minimum.

Having defined the notion of flat local minima, we define an approximate version of them such that we can discuss the iteration complexity of finding them. 

\begin{definition} [{\bf $(\eps,\eps')$-flat local minima}] \label{def:flat_min}
We say a point $\x$ is an $(\eps,\eps')$-flat local minimum if for $\x^\star = \Phi(\x)$, it satisfies
\begin{align}
\norm{\x-\x^\star}\leq \eps \quad \text{and} \quad \norm{ \left[\nabla\left(\ttr \circ \Phi\right) \right] (\x^\star) }  \leq \eps'\,.
\end{align}
In other words, a $(0,0)$-flat local minimum is a flat local minimum.
\end{definition}

\section{Randomly Smoothed Perturbation Escapes Sharp Minima} 
\label{sec:RS}

In this section, we present a gradient-based algorithm for finding an approximate flat minimum. We first discuss the setting for our analysis.

In order for our notion of flat minima  (\autoref{def:flat_min})  to be well defined, we assume that the loss function is four times continuously differentiable near the local minima set $\Xstar$.
More formally, we make the following assumption about the loss function.

\begin{restatable}[Loss near minima]{assumption}{aslocal} \label{as:local}
There exists $\zeta>0$ such that  within $\zeta$-neighborhood of the set of local minima $\Xstar$, the following properties hold:
\begin{enumerate}
\item[(a)] $f$ is four-times continuously differentiable.
\item[(b)] The limit map under gradient flow $\Phi$ (\autoref{def:Phi}) is well-defined and is twice Lipschitz differentiable. Also, $\Phi(\x) \in \Xstar$ and  the gradient flow starting at $\x$ is contained within the $\zeta$-neighborhood of $\Xstar$.
\item[(c)] The  Polyak--\L{}ojasiewicz (PL) inequality  holds locally, i.e., $ f(\x)-f(\Phi(\x)) \leq \frac{1}{2\alpha } \norm{\nabla f(\x)}^2$.
\end{enumerate}
\end{restatable}
It fact, the last two conditions (b), (c) are  consequences of $f$ being four-times continuously differentiable~\citep[Appendix B]{arora2022understanding}.
We include them for concreteness.

We also discuss a preliminary step for our analysis.
Since the question of finding candidates for approximate local minima (or second order stationary points) is well-studied, thanks to the  vast literature on the topic over the last decade~\citep{ge2015escaping,agarwal2017finding,carmon2018accelerated,fang2018spider, jin2021nonconvex}, we do not further explore it, but single out the question of seeking flatness by assuming that the initial iterate $\x_0$ is already close to the set of local minima $\Xstar$.
For instance, assuming that the loss $f$ satisfies strict saddle properties~\citep{ge2015escaping, jin2017escape},
one can find a point $\x_0$ that satisfies $\norm{\nabla f(\x_0)} \le \oo{\eps}$ within $\tilde{\mathcal{O}}(\eps^{-2})$iterations.
Now thanks to \autoref{as:local}, since we assume $f$ to be four-times continuously differentiable, it follows that $\norm{\x_0 -\Phi(\x_0)}\leq \oo{\norm{\nabla f(\x_0)} } \leq \oo{\eps}$.
Hence, we will often start our analysis with the initialization that is sufficiently close to the set of local minima $\Xstar$. 

We also define the following notation, which we will utilize throughout.
\begin{definition}[Projecting-out operator]
For two vectors ${\bm u},{\bm v}$,  $\proj^\perp_ {\bm u} {\bm v}$ is the ``projecting-out'' operator, i.e., 
\begin{align}
\proj^\perp_{\bm u} {\bm v} \coloneqq {\bm v} - \inp{\frac{\bm u}{\norm{{\bm u}}}}{{\bm v}} \frac{{\bm u}}{\norm{{\bm u}}}\,.
\end{align}  
\end{definition}

\subsection{Main Result}

Under this setting, we present  a gradient-based algorithm for finding approximate flat minima and its theoretical guarantees.
Our proposed algorithm is called  the randomly smoothed perturbation algorithm (\autoref{algo:RS}).
The main component of the algorithm is the perturbed gradient step that is employed whenever the gradient norm is smaller than a tolerance $\eps_0$:
\begin{align} \label{eq:perturb step}
\x_{t+1} &\leftarrow \x_t - \eta \left(\nabla f (\x_t) + \vv_t  \right), \quad \text{where}\\
\vv_t &\coloneqq \proj^\perp_{\nabla f(\x_t)} \nabla f(\x_t + \rho  \per_t )
\end{align} 
Here $\per_t \sim\mathrm{Unif}(\mathbb{S}^{d-1})$ is a random unit vector.
At a high level, \eqref{eq:perturb step} adds a perturbation direction $\vv_t$ to the ordinary gradient step, where the perturbation direction $\vv_t$ is computed using gradients at a randomly perturbed iterate $\x_t + \rho  \per_t$ and then projecting out the gradient $\nabla f(\x_t)$.
The gradient of a randomly perturbed iterate $\nabla f(\x_t + \rho  \per_t)$ can be also interpreted as the (stochastic) gradient of  widely known \emph{randomized smoothing}  of $f$ (hence its name ``randomly smoothed perturbation'')---a widely known technique for nonsmooth optimization~\citep{duchi2012randomized}. 
In some sense, this work discovers a new property of randomized smoothing for nonconvex optimization: \emph{randomized smoothing seeks flat minima!}
We now present the theoretical guarantee of \autoref{algo:RS}. 

\begin{algorithm}
\caption{Randomly Smoothed Perturbation  }\label{algo:RS}
\begin{algorithmic}
\renewcommand{\algorithmicrequire}{\textbf{Input: }}
\renewcommand{\algorithmicensure}{\textbf{Output: }}
\REQUIRE {$\x_0$, learning rates $\eta, \eta'$, perturbation radius $\rho$, tolerance $\eps_0$, the number of steps $T$.} 
\FOR{$t = 0, 1, \ldots, T-1$}
\IF{$\norm{\nabla f(\x_t)} \le \epsilon_0$}
\STATE $\x_{t+1} \leftarrow \x_t - \eta \left(\nabla f (\x_t) + \vv_t  \right)$, 
\STATE \quad\quad\quad \quad where $\vv_t \coloneqq \proj^\perp_{\nabla f(\x_t)} \nabla f(\x_t + \rho  \per_t )$ and $\per_t \sim \mathrm{Unif}(\mathbb{S}^{d-1})$
\ELSE
\STATE $\x_{t+1} \leftarrow \x_t - \eta'  \nabla f (\x_t)    $
\ENDIF
\ENDFOR 
\RETURN $\xh$ uniformly at random\footnotemark{} from $\{\x_1, \dots, \x_T \}$
\end{algorithmic}
\end{algorithm}  
\footnotetext{The uniformly chosen iterate is for the sake of analysis, and it's a standard approach often used in convergence to stationary points analysis. See, \emph{e.g.}, \citep{ghadimi2013stochastic,reddi2016stochastic}.}
\begin{theorem} \label{thm:RS}
Let \autoref{as:local} hold and $f$ have $\beta$-Lipschitz gradients.  Let the target accuracy $\eps>0$ be chosen sufficiently small, and  $\delta\in(0,1)$.
Suppose that $\x_0$ is $\zeta$-close to $\Xstar$.
Then,  the randomly smoothed perturbation algorithm  (\autoref{algo:RS}) with parameters  $\eta =\oo{\delta\eps}$, $\eta' =1/\beta$, $\rho=\oo{\delta \sqrt{\eps}}$,  $\eps_0=\oo{\delta^{1.5}\eps}$  returns an $(\eps,\sqrt{\eps})$-flat minimum with probability at least $1-
\oo{\delta}$ after $T=\oo{\eps^{-3}\delta^{-4}}$ iterations.    
\end{theorem}

\paragraph{Minimizing flatness only using gradients?}
At first glance, finding a flat minimum seems  computationally expensive since minimizing  $\tr(\nabla^2 f)$ would require information about second or higher derivatives. 
Thus, \autoref{thm:RS} may sound quite surprising to some readers since \autoref{algo:RS} only uses gradients which only pertains to information about first derivatives. 

However, it turns out using the gradients from the perturbed iterates $\x_t +\rho\per_t$ lets us get access to specific third derivatives of $f$ in a parsimonious way. 
More precisely, as we shall see in our proof sketch, the crux of  the perturbation step \eqref{eq:perturb step} is that the gradients of $\ttr$ can be estimated using  gradients from perturbed iterates.
In particular, we show that (see \eqref{eq:trace_grad}) in expectation, it holds that
\begin{align}
\E \vv_t = \frac{1}{2}\rho^2  \proj^\perp_{\nabla f(\x_t)} \nabla \ttr (\x_t)  + \text{lower order terms.}
\end{align}
Using this property, one can prove that each step of the perturbed gradient step decrease the trace of Hessian along the local minima set; see \autoref{lem:trace_descent}.
We remark that this general principle of estimating higher order derivatives from gradients in a parsimonious way is inspired by recent works on understanding dynamics of sharpness-aware minimization~\citep{bartlett2022dynamics,wen2022does} and gradient descent at edge-of-stability~\citep{arora2022understanding,damian2022self}.

\subsection{Proof Sketch of \autoref{thm:RS}}
\label{sec:pf_RS}
In this section, we provide a proof sketch of \autoref{thm:RS},
The full proof can be found in \autoref{pf:thm:RS}.
We first sketch the overall structure of the proof and then detail each part:
\begin{enumerate}
\item We first show that iterates enters  an $\oo{\eps_0}$-neighborhood of the local minima set $\Xstar$ in a few steps, and the subsequent iterates remain there.
\item When the iterates is $\oo{\eps_0}$-near $\Xstar$, we show that the perturbed gradient step in \autoref{algo:RS} decreases the trace of Hessian $\ttr(\Phi)$ in expectation as long as $\norm{ \partial \Phi( \Phi(\x_t)) \nabla \ttr( \Phi(\x_t)) }\geq \sqrt{\eps}$.

\item We then combine the above two properties to show that \autoref{algo:RS} finds a flat minimum.
\end{enumerate} 

\paragraph{Perturbation does not increase the cost too much.}
First, since $\x_0$ is $\zeta$-close to $\Xstar$ where the loss function satisfies the  Polyak--\L{}ojasiewicz (PL) inequality, the standard linear convergence result of gradient descent guarantees that the iterate enters an $\oo{\eps_0}$-neighborhood of $\Xstar$.
We thus assume that $\x_0$ itself satisfies $\norm{\nabla f(\x_0)}\leq \eps_0$ without loss of generality.
We next show that the perturbation $\vv_t$ we add at each step to the gradient only leads to a small increase in the cost. This claim follows from the following variant of well-known descent lemma.

\begin{lemma} \label{lem:descent_gd}
For $\eta \leq  1/\beta$,  consider a one-step of the perturbed gradient step of \autoref{algo:RS}: $\x_{t+1} \leftarrow \x_t - \eta \left(\nabla f (\x_t) + \vv_t  \right)$.
Then we have 
\begin{align}
\loss(\x_{t+1}) \le \loss(\x_t) -\frac{1}{2}\eta   \norm{\nabla \loss(\x_t)}^2 +\frac{\beta\eta^2}{2}\norm{\vv_t}^2.
\end{align}
\end{lemma}  
The proof of \autoref{lem:descent_gd} uses the fact that $\vv_t \perp \nabla f(\x_t)$. 
Now, with the $\beta$-Lipschitz gradient condition, one can show that  $\norm{\vv_t} =\oo{\rho}$.
Hence, whenever the gradient becomes large as $\norm{\nabla f(\x_t) }\gtrsim \eta \rho^2$, the perturbed update starts decreasing the loss again and brings the iterates back close to $\Xstar$.
Using this property, one can show that the iterates $\x_t$ remain in an $\eps_0$-neighborhood of $\Xstar$, i.e.,  $\norm{\x_t-\Phi(\x_t)} = \oo{\norm{\nabla f(\x_t)}}= \oo{\eps_0}$. See \autoref{lem:stay_near} for precise details.

\paragraph{Perturbation step decreases $\ttr(\Phi(\x_t))$ in expectation.} Now the main part of the analysis is to show that the perturbation updates lead to decrease in the trace Hessian along $\Xstar$, i.e., decrease in $\ttr(\Phi(\x_t))$, as show in the following result.
\begin{lemma} \label{lem:trace_descent}
Let \autoref{as:local} hold. Let the target accuracy $\eps>0$ be chosen sufficiently small, and  $\delta\in(0,1)$.
Consider the perturbed gradient step of \autoref{algo:RS}, i.e.,
$\x_{t+1} \leftarrow \x_t - \eta \left(\nabla f (\x_t) + \vv_t  \right)$ starting from $\x_t$ such that $\norm{\nabla f(\x_t)} \leq \eps_0$ with parameters $\eta= \oo{\delta\eps}$, $\rho = \oo{\delta \sqrt{\eps}}$ and $\eps_0=\oo{\delta^{1.5}\eps}$. 
Assume that $\ttr( \Phi(\x_t)) $ has sufficiently large gradient 
\begin{align}  
\norm{ \partial \Phi( \Phi(\x_t)) \nabla \ttr( \Phi(\x_t)) }  \geq \sqrt{\eps}\,.
\end{align}
Then  the trace of Hessian $\ttr(\Phi)$ decreases as 
\begin{align} \label{ineq:trace_descent}
\E\ttr(\Phi(\x_{t+1})) - \ttr(\Phi(\x_t))  &\leq  - \Omega(\delta^3\eps^3)\,, 
\end{align}
where the expectation is over the perturbation $\per_t \sim \mathrm{Uniform}(\mathbb{S}^{d-1})$ in \autoref{algo:RS}.
\end{lemma} 


\emph{Proof sketch of \autoref{lem:trace_descent}:}
We begin with the Taylor expansion of the perturbed gradient:
\begin{align} \label{eq:taylor}
\nabla f(\x_t + \rho \per_t)  &=  \nabla f(\x_t)  + \rho \nabla^2 f(\x_t) \per_t \\
&+ \frac{1}{2}\rho^2 \nabla^3 f(\x_t) \left[\per_t,\per_t\right] + \oo{\rho^3}\,.
\end{align}
Now let us compute the expectation of the projected out version of the perturbed gradient, i.e., $\E \proj^\perp_{\nabla f(\x_t)} \nabla f(\x_t + \rho \per_t)$.
First, note that in  \eqref{eq:taylor}, the projection operator removes $\nabla f(\x_t)$, and using the fact  $\E[\per_t]=\mathbf{0}$, the second term  $\rho \nabla^2 f(\x_t) \per_t$ also vanishes in expectation.
Turning to the third term, an interesting thing happens. 
Since $\E[\per_t\per_t^\top]= \frac{1}{d}\mathbf{I}_d$, using the fact $\nabla^3 f(\x_t) \left[\per_t,\per_t \right] = \nabla (\nabla^2 f(\x_t) \left[\per_t,\per_t \right]) = \nabla  \tr\left(\nabla^2 f(\x_t) \per_t \per_t^\top\right)$, it follows that 
\begin{align} \label{eq:trace_grad}
\E \vv_t   &= \frac{1}{2}\rho^2  \proj^\perp_{\nabla f(\x_t)} \nabla \ttr (\x_t) + \oo{\rho^3}\,, 
\end{align}  
Now, with the high-order smoothness properties of $f$, we obtain
\begin{align}  
&\ttr(\Phi(\x_{t+1})) - \ttr(\Phi(\x_t))\\
&\quad= \ttr(\Phi(\Phi(\x_{t+1}))) - \ttr(\Phi(\Phi(\x_t))) \\
&\quad \leq \inp{ \partial \Phi( \Phi(\x_t)) \nabla \ttr( \Phi(\x_t))}{\Phi(\x_{t+1}) -\Phi(\x_t)} \\
&\quad\quad+ \oo{\norm{\Phi(\x_{t+1}) -\Phi(\x_t)}^2}\,.
\end{align}
Using \eqref{eq:trace_grad} and carefully bounding terms, one can prove the following upper bound on $\E\ttr(\Phi(\x_{t+1})) - \ttr(\Phi(\x_t))$:   ($\partr \coloneqq \partial \Phi( \Phi(\x_t)) \nabla \ttr( \Phi(\x_t))$)
\begin{align} \label{ineq:trace_descent_RS}
-  \eta\rho^2\norm{\partr}^2   +   O \left(\eta\rho^2\eps_0 \norm{\partr}+ \eta\rho^3 \norm{\partr} +  \eta^2  \rho^2\right)\,. 
\end{align}
The inequality \eqref{ineq:trace_descent} implies that as long as $\norm{\partr}\geq \Omega(\max\{\eps_0, \rho, \sqrt{\eta}\})$,  $\ttr(\Phi(\x_t))$ decreases in expectation by $\eta\rho^2 \norm{\partr}^2$.
Due to our choices of $\rho,\eta,\eps_0$, \autoref{lem:trace_descent} follows. \qed

Using similar argument, one can show that the perturbation step does not increase the trace Hessian value too much even when   $\norm{ \partial \Phi( \Phi(\x_t)) \nabla \ttr( \Phi(\x_t)) }  \leq \sqrt{\eps}$. 
\begin{lemma} \label{lem:trace_descent_small}
Under the same setting as \autoref{lem:trace_descent}, assume now that  $\norm{ \partial \Phi( \Phi(\x_t)) \nabla \ttr( \Phi(\x_t)) }  \leq \sqrt{\eps}$. Then we have $    \E\ttr(\Phi(\x_{t+1})) - \ttr(\Phi(\x_t))  \leq  \oo{ \delta^4\eps^3 }$.
\end{lemma} 

\paragraph{Putting things together.}   Using the results so far, we establish a high probability result by returning one of the iterates uniformly at random, following \citep{ghadimi2013stochastic,reddi2016stochastic,daneshmand2018escaping}.
For $t=1,2,\dots, T$,  
\begin{align}
\text{let $A_t$ be the event $\norm{ \partial \Phi( \Phi(\x_t)) \nabla \ttr( \Phi(\x_t)) } \geq \sqrt{\eps}$,}
\end{align}
and Let $P_t$ denote the probability of event $A_t$. 
Then, the probability of returning a $(\eps,\sqrt{\eps})$-flat minimum is simply equal to $\frac{1}{T}\sum_{t=1}^T (1- P_t)$.
It turns out  one can upper bound the sum of $P_t$'s using \autoref{lem:trace_descent}; see \autoref{pf:thm:RS} for details.
In particular, choosing  $T =\om{\eps^{-3}\delta^{-4}}$, we get
\begin{align}  
\frac{1}{T}\sum_{t=1}^T P_t \lesssim  \frac{\E\ttr(\Phi(\x_0))}{T\delta^3\eps^3} +   \delta  = \oo{\delta}\,.
\end{align}
This concludes the proof of \autoref{thm:RS}.

\section{Faster Escape with Sharpness-Aware Perturbation}
\label{sec:SAM}

In this section, we present another gradient-based algorithm for finding an approximate flat minima for the case where the loss $f$ is a training loss over a training data set.
More formally, we consider the following setting for training loss, following the one in 
\citep{wen2022does}.

\begin{setting}
[Training loss over data] \label{def:ER}
Let $n$ be the number of training data, and for $i=1,\dots, n$, let $\mm_i(\x)$ be the model prediction output on the $i$-th data, and $y_i$ be the $i$-th label.
For a loss function $\ell$, let $f$ be defined as the following training loss 
\begin{align}
f(\x) = \frac{1}{n}\sum_{i=1}^n f_i(\x) \coloneqq \frac{1}{n}\sum_{i=1}^n \ell(\mm_i(\x), y_i)\,.
\end{align}
Here $\ell$ satifies $\argmin_{z\in\R}\ell(z, y) = y$ $\forall y$, and    $\frac{\partial^2\ell(z,y)}{\partial^2 z}\vert_{z=y} >0$.
Lastly, we consider $\Xstar$ to be the set of global minima, i.e., $\Xstar =\{ \x\in\R^d ~:~ \mm_i(\x) = y_i, \forall i=1,\dots,n  \}$. 
We assume that $\nabla \mm_i(\x)\neq \mathbf{0}$ for $\x\in \Xstar$, $\forall i=1,\dots,n$.  
\end{setting}
We note that  the assumption that $\nabla \mm_i(\x)\neq \mathbf{0}$ for $\x\in \Xstar$ is without loss of generality. 
More precisely, by Sard's Theorem, $\Xstar$ defined above is just equal to the set of global minima, except for a measure-zero set of labels.

\subsection{Main Result}

Under \autoref{def:ER}, we present another gradient-based algorithm for finding approximate flat minima (\autoref{algo:SAM}).
The main component of our proposed algorithm is   the perturbed gradient step
\begin{align} \label{eq:perturb step_sam}
\x_{t+1} &\leftarrow \x_t - \eta \left(\nabla f (\x_t) + \vv_t  \right)\,,\quad \text{where}\\
\vv_t &\coloneqq \proj^\perp_{\nabla f(\x_t)} \nabla f_i\left(\x_t + \rho \sigma_t \frac{\nabla f_i(\x_t)}{  \norm{\nabla f_i(\x_t) }}\right) 
\end{align}
for random samples $i \sim [n]$ and $\sigma_t \sim \{\pm 1\}$.

\begin{remark} \label{rmk:well_define}
Here, note that the direction $\nicefrac{\nabla f_i(\x_t)}{  \norm{\nabla f_i(\x_t) }}$ could be ill-defined when the stochastic gradient exactly vanishes at $\x_t$.
In that case, one can use $\nicefrac{\nabla f_i(\x_t + {\bm\xi})}{  \norm{\nabla f_i(\x_t+ {\bm \xi}) }}$ where $\bm \xi$ is a random vector with a small norm, say $\eps^{3}$. 
Hence, to avoid tedious technicality, we assume for the remaining of the paper that \eqref{eq:perturb step_sam} is well-defined at each step.    
\end{remark}

Notice the distinction between \eqref{eq:perturb step} and \eqref{eq:perturb step_sam}. 
In particular,  for the randomly smoothed perturbation algorithm, $\vv_t$ is computed using the gradient at a randomly perturbed iterate.
On the other hand, in the update \eqref{eq:perturb step_sam}, $\vv_t$ is computed using the \emph{stochastic} gradient at an iterate perturbed along the \emph{stochastic gradient direction}.
The idea of computing the (stochastic) gradient at an iterate perturbed along the (stochastic) gradient direction is inspired by sharpenss-aware minimization (SAM) of \citet{foret2020sharpness}, a practical optimization algorithm showing substantial success in practice. Hence, we call our algorithm the \emph{sharpness-aware perturbation} algorithm.

As we shall see in detail in \autoref{thm:SAM}, the sharpness-aware perturbation step \eqref{eq:perturb step_sam} leads to an improved guarantee for finding a flat minimum. 
The key idea---as we detail in \autoref{sec:pfsketch_sam}---is that this perturbation leads to faster decrease in $\ttr$.
In particular, \autoref{lem:trace_descent_sam} shows that each sharpness-aware perturbation decreases $\ttr$ by $\Omega( {d  \min\{ 1,\eps d^3\}}\cdot \delta^3\eps^2)$, which is 
$d\eps^{-1}  \min\{ 1, \eps d^3\}$ times larger than the decrease of $\Omega(\delta^3 \eps^3)$ due to the randomly smoothed perturbation (shown in \autoref{lem:trace_descent}).
We now present the theoretical guarantee of \autoref{algo:SAM}.

\begin{algorithm}[t]
\caption{Sharpness-Aware Perturbation}\label{algo:SAM}
\begin{algorithmic}
\renewcommand{\algorithmicrequire}{\textbf{Input: }}
\renewcommand{\algorithmicensure}{\textbf{Output: }}
\REQUIRE {$\x_0$, learning rates $\eta,\eta'$, perturbation radius $\rho$, tolerance $\eps_0$, the number of steps $T$.} 
\FOR{$t = 0, 1, \ldots, T-1$}
\IF{$\norm{\nabla f(\x_t)} \le \eps_0$}
\STATE $\x_{t+1} \leftarrow \x_t - \eta \left(\nabla f (\x_t) + \vv_t  \right)$, 
\STATE  where $\vv_t \coloneqq \proj^\perp_{\nabla f(\x_t)} \nabla f_i(\x_t + \rho \sigma_t \frac{\nabla f_i(\x_t)}{  \norm{\nabla f_i(\x_t) }})$ for $i\sim \mathrm{Unif}([n])$ and $\sigma_t\sim \mathrm{Unif}(\{\pm 1\})$.
\ELSE
\STATE $\x_{t+1} \leftarrow \x_t - \eta' \nabla f (\x_t)    $
\ENDIF  
\ENDFOR 
\RETURN  $\widehat{\x}$ uniformly at random from $\{\x_1, \dots, \x_T \}$
\end{algorithmic}
\end{algorithm}

\begin{theorem} \label{thm:SAM}
Under \autoref{def:ER}, let \autoref{as:local} hold and each $f_i$ is four times coutinuously differentiable within the $\zeta$-neighborhood of $\Xstar$ and have $\beta$-Lipschitz gradients. 
Let the target accuracy $\eps>0$ be chosen sufficiently small, and  $\delta\in(0,1)$.
Suppose that $\x_0$ is $\zeta$-close to $\Xstar$.
Then, for $\nu \coloneqq \min\{d,\eps^{-1/3}\}$, the sharpness-aware perturbation algorithm  (\autoref{algo:SAM}) with parameters 
$\eta =\oo{\nu\delta\eps}$, $\eta' =1/\beta$, $\rho=\oo{\nu\delta \sqrt{\eps}}$, $\eps_0=\oo{  \nu^{1.5}\delta^{1.5}\eps}$ returns an $(\oo{\eps_0},\sqrt{\eps})$-flat minimum with probability at least $1-
\oo{\delta}$ after $T  = \oo{d^{-1}\eps^{-2}  \cdot \max\left\{1,\frac{1}{d^3\eps}\right\} \cdot \delta^{-4} }$ iterations. 
From this $(\oo{\eps_0},\sqrt{\eps})$-flat minimum,  gradient descent with step size $\eta=\oo{\eps}$ reaches a $(\eps,\sqrt{\eps})$-flat minimum within  $\oo{\eps^{-1}\log (1/\eps)}$ iterations.
\end{theorem}

\paragraph{Curious role of stochastic gradients.} Some readers might wonder the role of stochastic gradients in \eqref{eq:perturb step_sam}---for instance, what happens if we replace them by \emph{full-batch} gradients $\nabla f$?
Empirically, it has been observed that for SAM's performance, it is important to use stochastic gradients over full-batch~\citep{foret2020sharpness,kaddour2022flat,kaur2023maximum}.
Our analysis (see the proof sketch of \autoref{lem:trace_descent_sam}) provides a partial explanation for the success of using stochastic gradients, from the perspective of finding flat minima.
In particular, we show that stochastic gradients are important for faster decrease in the trace of the Hessian.

\subsection{Proof Sketch of \autoref{thm:SAM}}
\label{sec:pfsketch_sam}

\begin{figure*}
\centering
\includegraphics[width=0.7
\textwidth]{figs/cifar10.pdf} 
\vspace{-10pt}
\caption{ (Left) The comparison between Randomly Smoothed Perturbation (``RS'') and  Sharpness-Aware Perturbation (``SA''). (Right) Comparison of SA with different batch sizes. Here, we highlight that we do observe that {\bf the trace of Hessian value monotonically decreases} along the algorithm iterates, similarly to \citep{damian2021label} (see also \autoref{fig:damian}). We decide to present the test accuracy instead of the trace of Hessian, as it has more practical values.
} \label{fig:cifar10}

\end{figure*}

In this section, we sketch a proof of \autoref{thm:SAM}; for the full proof please see~\autoref{pf:thm:SAM}.
First, similarly  to the proof of \autoref{thm:RS}, one can show that once $\x_t$ enters an $\oo{\eps_0}$-neighborhood of $\Xstar$, all subsequent iterates $\x_t$ remain in the neighborhood. Now we sketch the proof of decrease in the trace of the Hessian.

\paragraph{Sharpness-aware perturbation decreases $\ttr(\Phi(\x_t))$ faster.} Similarly to \autoref{lem:trace_descent}, the main part  is to show that  the trace of Hessian decreases during each perturbed gradient step.

\begin{lemma} \label{lem:trace_descent_sam}
Let \autoref{as:local} hold. 
Let the target accuracy $\eps>0$ be chosen sufficiently small, and  $\delta\in(0,1)$.
Consider the perturbed gradient step of \autoref{algo:SAM}, i.e.,
$\x_{t+1} \leftarrow \x_t - \eta \left(\nabla f (\x_t) + \vv_t  \right)$ starting from $\x_t$ such that $\norm{\nabla f(\x_t)} \leq \eps_0$ with parameters $\eta = \oo{\nu\delta\eps}$,  $\rho=(\nu \delta\sqrt{\eps})$, $\eps_0=\oo{  \nu^{1.5}\delta^{1.5}\eps}$. 
Assume that $\ttr( \Phi(\x_t))$ has sufficiently large gradient 
\begin{align}  
\norm{ \partial \Phi( \Phi(\x_t)) \nabla \ttr( \Phi(\x_t)) }  \geq \sqrt{\eps}\,.
\end{align}
Then the  trace of Hessian $\ttr(\Phi)$ decreases as 
\begin{align}  
\E\ttr(\Phi(\x_{t+1})) - \ttr(\Phi(\x_t))  &\leq  - \Omega(  \textcolor{red}{d\nu^3}\delta^3\eps^3)\,, 
\end{align}
where the expectation is over the random samples $i\sim \mathrm{Unif}([n])$ and $\sigma_t\sim \mathrm{Unif}(\{\pm 1\})$   in \autoref{algo:SAM}.
\end{lemma} 

\emph{Proof sketch of \autoref{lem:trace_descent_sam}:}
For notational simplicity, let $\per_{i,t} \coloneqq   \frac{\nabla f_i(\x_t)}{  \norm{\nabla f_i(\x_t) }}$.
To illustrate the main idea effectively, we make the simplifying assumption that for $\x^\star \in\Xstar$, the gradient of model outputs $\{\nabla \mm_i(\x)\}_{i=1}^n$ are orthogonal; our full proof in \autoref{pf:thm:SAM} does not require this assumption.
To warm-up, let us first consider the case where we use the \emph{full-batch gradient} instead of the stochastic gradient for the outer part of the perturbation, i.e., consider
\begin{align}
\vv_t' \coloneqq \proj^\perp_{\nabla f(\x_t)} \nabla f(\x_t + \rho \sigma_t \per_{i,t} )
\end{align}
Because  $\E[\sigma_t \per_t]=\mathbf{0}$ (since 
$\E[\sigma_t]=0$), a similar calculation as the proof of \autoref{lem:trace_descent},  we arrive at 
\begin{align}
\E \vv_t'   &= \frac{1}{2}\rho^2  \proj^\perp_{\nabla f(\x_t)} \nabla \tr\left(\nabla^2 f(\x_t) \per_{i,t} \per_{i,t}^\top\right) + \oo{\rho^3}\,.
\end{align}
Now the key observation, inspired by \citet{wen2022does}, is that at a minimum $\x^\star \in \Xstar$, the Hessian is given as 
\begin{align}
\nabla^2 f(\x^\star) =   \frac{1}{n}  \sum_{i=1}^n  \ell''(\mm_i(\x^\star),y_i)  \nabla \mm_i(\x^\star) \nabla \mm_i(\x^\star)^\top \,.
\end{align}
Hence, due to our simplifying assumption for this proof sketch, namely the orthogonality of  the gradient of model outputs $\{\nabla \mm_i(\x^\star)\}_{i=1}^n$, it follows that $\pp_i(\x^\star)\coloneqq \nabla \mm_i(\x^\star)/\norm{\nabla \mm_i(\x^\star)}$ is the eigenvector of the Hessian.
Let  $\lambda_i(\x^\star)$ be the corresponding  eigenvalue. 
Furthermore, we have $\nabla f_i(\x_t) = \frac{\partial\ell(z,y_i)}{\partial z}\vert_{z=\mm_i(\x_t) } \nabla \mm_i (\x_t)$, which implies $\frac{\nabla f_i(\x_t)}{\norm{\nabla f_i(\x_t)}} =  \frac{\nabla \mm_i (\x_t)}{\norm{\nabla \mm_i (\x_t)}}= \pp_i(\x_t)$ as long as $\nabla f_i(\x_t)\neq 0$.
Hence, as long as $\x_t$ stays near $\Phi(\x_t)$, it follows that 
\begin{align}
\E\tr\left(\nabla^2 f(\x_t) \per_{i,t} \per_{i,t}^\top\right) \approx  \E \lambda_i(\Phi(\x_t))  =   \frac{d}{n}\ttr(\Phi(\x_t)) \,,
\end{align}
which notably gives us $\frac{d}{n}$ times larger gradient than the randomly smoothed perturbation \eqref{eq:trace_grad}.
On the other hand, one can do even better by choosing the stochastic gradient for the outerpart of perturbation. Similar calculations to the above yield
\begin{align}
&\E\tr\left(\nabla^2 f_i(\x_t) \per_{i,t} \per_{i,t}^\top\right)  \approx  n  \E \lambda_i(\Phi(\x_t))   = \textcolor{red}{d}\cdot \ttr(\Phi(\x_t)) \,,
\end{align}
which now leads to $d$ times larger gradient than  \eqref{eq:trace_grad}. This leads to the following inequality that is an improvement of \eqref{ineq:trace_descent_RS}:   $\E\ttr(\Phi(\x_{t+1})) - \ttr(\Phi(\x_t))  -  \textcolor{red}{d}\cdot \eta\rho^2\norm{\partr}^2 +   \oo{d\eta\rho^2 \eps_0 \norm{\partr}+ \eta\rho^3 \norm{\partr} +  \eta^2  \rho^2}$.
This inequality  implies that as long as $\norm{\partr}\geq \Omega(\max\{\eps_0, \rho/d, \sqrt{\eta/d}\})$,  $\ttr(\Phi(\x_t))$ decreases in expectation by $d\eta\rho^2 \norm{\partr}^2$.
Due to our choices of $\rho,\eta,\eps_0$, \autoref{lem:trace_descent_sam} follows.   \qed

Using \autoref{lem:trace_descent_sam}, and following the analysis presented in \autoref{sec:pf_RS}, it can be shown that  \autoref{algo:SAM}
returns a  $(\eps_0,\sqrt{\eps})$-flat minimum $\widehat{\x}$ with probability at least $1-\oo{\delta}$ after 
$T =\oo{d^{-1}\eps^{-3} \nu^{-3} \delta^{-4} }$ iterations.
From this $(\eps_0,\sqrt{\eps})$-flat minimum $\widehat{\x}$, one can find a $(\eps,\sqrt{\eps})$-flat minimum in a few iterations.

\section{Experiments}

We run experiments based on training ResNet-18 on the CIFAR10 dataset to test the ability of proposed algorithms to escape sharp global minima.
Following \cite{damian2021label}, the algorithms are initialized at a point corresponding to a sharp global minimizer that achieve poor test accuracy. 
Crucially, we choose this setting because \citep[Figure 1]{damian2021label} verify that test accuracy is inversely correlated with the trace of Hessian (see \autoref{fig:damian}). 
This bad global minimizer, due to \cite{liu2020bad},  achieves $100\%$ training accuracy, but only $48\%$ test accuracy. 
We choose the constant learning rate of $\eta=0.001$, which is small enough such that SGD baseline without any perturbation does not escape. 

We discuss the results one by one. First of all, we highlight that the training accuracy stays at $100\%$ for all algorithms.
\begin{itemize}
\item \textbf{Comparison between two methods.} In the left plot of  \autoref{fig:cifar10}, we compare the performance of Randomly Smoothed Perturbation (``RS'') and  Sharpness-Aware Perturbation (``SA''). We choose the batch size of $128$ for both methods. Consistent with our theory, one can see that SA is more effective in escaping sharp minima even with a smaller perturbation radius $\rho$.
\item \textbf{Different batch sizes.} Our theory suggests that batch size $1$ should be effective in escaping sharp minima. We verify this in the right plot of \autoref{fig:cifar10} by choosing the batch size to be $B=1,64,128$.
We do see that the case of $B=1$ is quite effective in escaping sharp minima.  
\end{itemize}

\section{Related Work and Future Work}

The last decade has seen a great success in theoretical studies on the question of  finding (approximate) stationary points~\citep{ghadimi2013stochastic,ge2015escaping,agarwal2017finding,daneshmand2018escaping, carmon2018accelerated,fang2018spider,allen2018natasha, zhou2020stochastic, jin2021nonconvex}.
This work extends this line of research to a new notion of stationary point, namely an approximate flat minima.
We believe that further studies on defining/refining practical notions of flat minima and designing efficient algorithms for them would lead to better understanding of practical nonconvex optimization for machine learning.
In the same spirit, we believe that characterizing lower bounds would be of great importance, similar to the ones for the stationary points \citep{carmon2020lower,drori2020complexity,carmon2021lower,arjevani2023lower}.

Another important direction is to further investigate the effectiveness of the flatness. As we discussed in \autoref{rmk:other}, recent results have shown that other notions of flatness are not always a good indicator of model efficacy \cite{andriushchenko2023modern,wen2023sharpness}. It would be interesting to understand the precise role of flatness, given that we have a lot of evidence of its success.
Moreover, studying other notions of flatness, such as the ``effective size of basin'' as considered in \cite{kleinberg2018alternative,feng2020dynamical},
or the constrained settings (\emph{e.g.}, \citep{feng2020dynamical}), and exploring the algorithmic questions there would be also interesting future directions. 

Based on our analysis, we suspect that replacing the full-batch gradients with the stochastic gradients in our proposed algorithms also leads to an efficient algorithm, with a more careful stochastic analysis.
Moreover, we suspect that our results have sub-optimal dependence on the error probability $\delta$, and a more advanced analysis will likely leads to a better dependence \citep{jin2021nonconvex}.
Lastly, based on our experiments, it seems that a smaller batch size has the same effect as using a larger perturbation radius $\rho$. 
Whether one can capture this effect theoretically would be also an intriguing direction.
However, the main scope of this work is to initiate the study of complexity of finding flat minima, and we leave all of this to future works.

\section*{Acknowledgements}

Kwangjun Ahn and Ali Jadbabaie were supported by the ONR grant (N00014-20-1-2394) and MIT-IBM Watson as well as a Vannevar Bush fellowship from Office of the Secretary of Defense.	
Kwangjun Ahn and Suvrit Sra acknowledge support from an NSF CAREER grant (1846088), and NSF CCF-2112665 (TILOS AI Research Institute). 
Suvrit Sra also thanks the Alexander von Humboldt Foundation for their generous support.

Kwangjun Ahn thanks Xiang Cheng, Yan Dai, Hadi Daneshmand, and Alex Gu for helpful discussions that led the author to initiate this work.

\section*{Impact Statement}

This paper aims to advance our theoretical understanding of flat minima optimization. 
Our work is theoretical in nature, and we do not see any immediate potential societal consequences.

\bibliography{ref}
\bibliographystyle{plainnat}


\newpage 
\onecolumn

\appendix
\renewcommand{\appendixpagename}{\centering \LARGE Appendix}
\appendixpage
\startcontents[section]
\printcontents[section]{l}{1}{\setcounter{tocdepth}{2}}

\section{Preliminaries}

In this section, we present background information 
and useful lemmas for our analysis.
We start with  several notations and conventions for our analysis.
\begin{itemize}
\item We will highlight the dependence on the relevant quantities $\eps,\delta,d$ and will often hide the dependence on other parameters in the notations $\oo{\cdot}, \tth{\cdot}, \om{\cdot}$.

\item We will sometimes abuse our notation as follows:  when the two vectors $\bm{u}, \bm{v}$ satisfy $\norm{\bm{u}-\bm{v}} = \oo{g(\eps,\delta,d)}$ for some function $g$ of $\eps,\delta,d$, then we will simply write
\begin{align}
\bm{u} = \bm{v} +\oo{g(\eps,\delta,d)} \,.
\end{align}
\item For a $\ell$-th order tensor $\mathcal{T} \in \R^{d_1 \times \cdots \times d_\ell}$, the spectral norm is defined as
\begin{align}
\norm{\mathcal{T}}_2 \coloneqq \sup_{{\bm u}_i \in \R^{d_i}, \norm{{\bm u}_i} =1} \mathcal{T}[{\bm u}_1,\dots, {\bm u}_\ell]\,.
\end{align}
\item For a tensor $\mathcal{T}(\x)$ that depends on $\x$ (e.g., $\nabla^2f(\x), \nabla^3 f(\x), \partial\Phi(\x)$ etc), let $\be{\mathcal{T}}$ be the upper bound on the spectrum norm  $\norm{\mathcal{T}(\x)}_2$ within the $\zeta$-neighborhood of $\Xstar$ ($\zeta$ is defined in \autoref{as:local}).
\end{itemize}

We  also recall our main assumption (\autoref{as:local}) for reader's convenience.
\aslocal*

\subsection{Auxiliary Lemmas}

We first present the following geometric result that  compares the cost, gradient norm, and the distance to $\Phi(\x)$ near $\Xstar$.

\begin{lemma} \label{lem:local}
Let
\autoref{as:local} hold and $f$ have $\beta$-Lipschitz gradients.
If $\x$ is in the $\zeta$-neighborhood of $\Xstar$, then it holds that
\begin{itemize}
\item $\norm{\x-\Phi(\x)} 
\leq \sqrt{\frac{2}{\alpha}} \sqrt{f(\x)-f(\Phi(\x))}$ and $\sqrt{f(\x)-f(\Phi(\x))} \leq \frac{\beta}{\sqrt{2\alpha}} \norm{\x -\Phi(\x)}$.
\item $\norm{\x-\Phi(\x)} \leq \frac{1}{\alpha} \norm{\nabla f(\x)}$ and $\norm{\nabla f(\x)} \leq \beta \norm{\x -\Phi(\x)}$.
\item $\sqrt{f(\x) - f(\Phi(\x))} \leq \frac{1}{\sqrt{2\alpha}} \norm{\nabla f(\x)}$ and $\norm{\nabla f(\x)} \leq \sqrt{\frac{2\beta^2}{\alpha}}\sqrt{f(\x)-f(\Phi(\x))}$.
\end{itemize} 
\end{lemma}
\begin{proof} See \autoref{pf:lem:local}.
\end{proof}
We next present an important property of the limit point under the gradient flow, $\Phi$.
\begin{lemma} \label{lem:Phi}
For any $\x$ at which $\Phi$ is defined and differentiable, we have that $\partial \Phi (\x) \nabla f(\x) =\mathbf{0}$. 
\end{lemma}
\begin{proof}
See \citep[Lemma 3.2]{wen2022does}  or \citep[Lemma C.2]{li2022what}.
\end{proof}

We next prove the following results about the distance in terms of $\Phi$ between two adjacent iterates.

\begin{lemma} \label{lem:Phi_diff}
Let \autoref{as:local} hold and $f$ have $\beta$-Lipschitz gradients.
For a vector $\vv$  satisfying   $\vv \perp \nabla f(\x)$ and $\norm{\vv} = \oo{1}$, consider the update $\x^+-\x = -\eta( \nabla f(\x) + \vv)$. Then, for suffciently small $\eta$, if $\x$ is in $\zeta$-neighborhood of $\Xstar$, the following holds:
\begin{itemize}
\item  $\Phi(x^+) - \Phi(x) =  - \eta \partial \Phi(x)  \vv +\oo{\be{\partial\Phi}\eta^2(\norm{\nabla f(\x)}^2+\norm{\vv}^2)}$.
\item  $\norm{{\Phi(\x^+) - \Phi(\x)}}^2 \leq 4 \be{\partial \Phi }^2 \eta^2\norm{\vv}^2 + 3\be{\partial\Phi}^2\eta^4\norm{\nabla f(\x)}^4 $.
\item  $|f(\Phi(\x^+)) - f(\Phi(\x))| = \oo{ (\be{\partial^2 \Phi} \be{\nabla f}+ \be{\partial \Phi} \be{\nabla^2 f})\be{\partial \Phi }^2 (\norm{\vv}^2 +\eta^2 \norm{\nabla f(\x)}^4) }$.
\end{itemize} 
\end{lemma}

\begin{proof} See \autoref{pf:lem:Phi_diff}.
\end{proof}

We next present the result about iterates staying near the local minima set $\Xstar$.

\begin{lemma} \label{lem:stay_near} 
Let \autoref{as:local} hold and $f$ have $\beta$-Lipschitz gradients.
For a vector $\vv$ satisfying   $\vv \perp \nabla f(\x)$ and $\norm{\vv} =\oo{1}$, consider the update
$\x^+ - \x = -\eta \cdot\left( \nabla \loss(\x) + {\vv}     \right)$. 
For sufficiently small $\eta>0$, we have the following: 
\begin{align}
\text{if $f(\x) - f(\Phi(\x))\leq \frac{2\beta}{\alpha} \eta \norm{\vv}^2$, then $f(\x^+) - f(\Phi(\x^+))\leq \frac{2\beta}{\alpha} \eta \norm{\vv}^2$ as well.}    
\end{align}

\end{lemma} 
\begin{proof} See \autoref{pf:lem:stay_near}.
\end{proof}


\section{Proof of \autoref{thm:RS}} \label{pf:thm:RS}

In this section, we present the proof of \autoref{thm:RS}.
The overall structure of the proof follows the proof sketch in \autoref{sec:pf_RS}.
We  consider the following choice of parameters for \autoref{algo:RS}: \begin{align} \label{exp:choice}
{   \eta =\tth{\frac{1}{\consta  \be{\partial \Phi }^2 \beta^2}\delta\eps },  ~~\rho=\tth{\frac{1}{\consta} \delta \sqrt{\eps}},~~ \eps_0= \tth{\frac{\beta^{3/2}}{\alpha \consta^{3/2} \be{\partial \Phi } } \delta^{1.5}\eps}}\,. 
\end{align}
where $\consta \coloneqq \be{\partial^2\Phi} \be{\nabla^3 f}+ \be{\partial \Phi}  \be{\nabla^4 f}$. 
Then, note that $\partial \Phi(\x) \nabla \ttr(\x)$ is $\consta$-Lipschitz in the $\zeta$-neighborhood of $\Xstar$.

First, since $\x_0$ is $\zeta$-close to $\Xstar$, the standard linear convergence  result of gradient descent for the cost function satisfying the  Polyak--\L{}ojasiewicz (PL) inequality   \citep{karimi2016linear} together with \autoref{lem:local} imply that with the step size $\eta' =\frac{1}{\beta}$, within $T_0= \oo{\log(1/\eps_0)}$ steps,
one can reach the point $\x_{T_0}$ satisfying
$\norm{\nabla f(\x_{T_0})}\leq \eps_0$.  Thus, we henceforth assume that $\x_0$ itself satisfies $\norm{\nabla f(\x_0)}\leq \eps_0$ without loss of generality.

Next, we show that for $\vv_t$ defined as in \autoref{algo:RS}, i.e., $\vv_t \coloneqq \proj^\perp_{\nabla f(\x_t)} \nabla f(\x_t + \rho  \per_t )$, we have
\begin{align} \label{eq:bound on v}
\text{$\norm{\vv_t} \leq \beta\rho$  at each step $t$.}
\end{align}
This holds because $\vv_t  =\proj^\perp_{\nabla \loss (\x_t)} (\nabla  \loss (\x_t + \rho \per_t) -  \nabla \loss(\x_t))$, 
and the ``projecting-out'' operator $\proj^\perp_{\nabla \loss (\x_t) }$ only decreases the norm of the vector: it follows that  $\norm{\vv_t}  \leq \norm{ \nabla  \loss (\x_t + \rho \per_t) -  \nabla \loss(\x_t)  } \leq \beta\rho$, as desired.

Then, by \autoref{lem:stay_near}, for sufficiently small $\eps$, it holds that  $f(\x_t) - f(\Phi(\x_t))\leq \frac{2\beta^3}{\alpha} \eta \rho^2$  during each step $t$. 
This implies together with \autoref{lem:local} that $\norm{\x_t -\Phi(\x_t)}^2\leq \frac{4\beta^{3}}{\alpha^2} \eta\rho^2$ and $\norm{\nabla f(\x_t)}^2\leq \frac{4\beta^{5}}{\alpha^2} \eta\rho^2$   during each step $t$.
Thus, due to the choice \eqref{exp:choice}, we conclude that
\begin{align} \label{eq:dist to opt}
\text{$\norm{\nabla f(\x_t)} \leq \eps_0~~\text{and}~~ \norm{\x_t -\Phi(\x_t)}\leq\oo{\eps_0}$\quad  hold during each step $t$.}
\end{align} 
We now characterize the direction $\partial \Phi(\x_t)  \vv_t$.
\begin{lemma} \label{lem:trace_gradient}
Let \autoref{as:local} hold and consider the parameter choice \eqref{exp:choice}. Then, for sufficiently small $\eps>0$, under the condition \eqref{eq:dist to opt}, $\vv_t$ defined in \autoref{algo:RS} satisfies 
\begin{align} 
\E\partial \Phi(\x_t)   \vv_t  =  \frac{1}{2}\rho^2 \partial\Phi(\Phi(\x_t)) \nabla \ttr(\Phi(\x_t)) +\oo{\consta \rho^3}\,.
\end{align}
\end{lemma}
\begin{proof}
See \autoref{pf:lem:trace_gradient}.
\end{proof} 

Using \autoref{lem:trace_gradient}, we can prove the following formal statement of \autoref{lem:trace_descent}.

\begin{lemma} \label{lem:trace_descent_formal}
Let \autoref{as:local} hold and choose the parameters as per \eqref{exp:choice}.
Let $\eps>0$ be chosen sufficiently small and $\delta\in(0,1)$.
Then, there exists an absolute constant $c>0$ s.t. the following holds:
if $\norm{\partial \Phi(\Phi(\x_t)) \nabla \ttr(\Phi(\x_t))}\geq \sqrt{\eps}$,  then
\begin{align}  
\E\ttr(\Phi(\x_{t+1})) - \ttr(\Phi(\x_t)) &\leq -  \frac{c}{\consta^3  \be{\partial \Phi }^2 \beta^2}\delta^3 \eps^3   \,.
\intertext{
On the other hand, if $\norm{\partial \Phi(\Phi(\x_t)) \nabla \ttr(\Phi(\x_t))}\leq \sqrt{\eps}$, then}   \E\ttr(\Phi(\x_{t+1})) - \ttr(\Phi(\x_t)) &\leq   \frac{c}{\consta^3  \be{\partial \Phi }^2 \beta^2}\delta^4 \eps^3  \,.
\end{align} 
\end{lemma} 
\begin{proof}
See \autoref{pf:lem:trace_descent_formal}.
\end{proof}

Now the rest of the proof follows the probabilistic argument in the proof sketch (\autoref{sec:pf_RS}). 
For $t=1,2,\dots, T$, let 
$A_t$ be the event where $\norm{ \partial \Phi( \Phi(\x_t)) \nabla \ttr( \Phi(\x_t)) } \geq \sqrt{\eps}$,
and let $R$ be a random variable equal to the ratio of desired flat minima visited among the iterates $\x_1,\dots,\x_T$.
Then,  
\begin{align}
R  = \frac{1}{T} \sum_{t=1}^T \mathds{1} \left( A_t^c \right)\,,
\end{align}
where $\mathds{1}$ is the indicator function.
Let $P_t$ denote the probability of event $A_t$. 
Then, the probability of returning a $(\eps,\sqrt{\eps})$-flat minimum is simply equal to $\E  R  = \frac{1}{T}\sum_{t=1}^T (1- P_t)$.
Now the key idea is that although estimating individual $P_t$'s might be difficult, one can upper bound the sum of $P_t$'s using \autoref{lem:trace_descent_formal}.
More specifically, \autoref{lem:trace_descent_formal} implies that 
\begin{align}   
\E\ttr(\Phi(\x_{t+1})) - \E\ttr(\Phi(\x_t))  &\leq - \frac{c}{\consta^3  \be{\partial \Phi }^2 \beta^2}\delta^3\eps^3\cdot (P_t -  \delta (1-P_t)) \\
& = \frac{c}{\consta^3  \be{\partial \Phi }^2 \beta^2}\delta^3\eps^3\cdot\left\{ \delta  - (1+ \delta) P_t \right\}\,, 
\end{align}
which after taking sum over $t=0 \dots, T-1$ and rearranging yields
\begin{align}  
\frac{1}{T}\sum_{t=1}^T P_t \leq   \frac{\consta^3  \be{\partial \Phi}^2\beta^2 }{cT\delta^3\eps^3} +   \delta   \,.
\end{align}
Hence choosing 
\begin{align}
T= \tth{\frac{\consta^3  \be{\partial \Phi}^2\beta^2 } {\eps^3\delta^4}}\,,
\end{align}   $\E R$ is lower bounded by $1-\oo{\delta}$, which concludes the proof of \autoref{thm:RS}.

\section{Proof of \autoref{thm:SAM}}
\label{pf:thm:SAM}

In this section, we present the proof of \autoref{thm:SAM}.
The overall structure of the proof is similar to that of \autoref{thm:RS}in \autoref{pf:thm:RS}.
We  consider the following choice of parameters for \autoref{algo:SAM}: for $\nu \coloneqq   \min\{d,\eps^{-1/3}\} $, \begin{align} \label{exp:choice_sam}
{    \eta =\tth{\frac{1}{\consta  \be{\partial \Phi }^2 \beta^2}\nu\delta\eps },  ~~\rho=\tth{\frac{1}{\consta} \nu\delta \sqrt{\eps}},~~ \eps_0= \tth{\frac{\beta^{3/2}}{\alpha \consta^{3/2} \be{\partial \Phi } } \nu^{3/2}\delta^{3/2} \eps}}\,. 
\end{align}
where this time 
we define $\consta \coloneqq \max_{i=1,\dots,n} \be{\partial^2\Phi} \be{\nabla^3 f_i}+ \be{\partial \Phi}  \be{\nabla^4 f_i}$.
Then, again note that $\partial \Phi(\x) \nabla \ttr(\x)$ is $\consta$-Lipschitz in the $\zeta$-neighborhood of $\Xstar$.

Again, similarly to the proof in \autoref{pf:thm:RS},  within $T_0= \oo{\log(1/\eps_0)}$ steps, one can reach $\x_{T_0}$ s.t. $\norm{\nabla f(\x_{T_0})}\leq \eps_0$, so we assume that $\x_0$ satisfies $\norm{\nabla f(\x_0)}\leq \eps_0$ without loss of generality.

We first show that for $\vv_t$ defined as $\proj^\perp_{\nabla f(\x_t)} \nabla f_i(\x_t + \rho \sigma_t \frac{\nabla f_i(\x_t)}{  \norm{\nabla f_i(\x_t) }})$, we have
\begin{align} \label{eq:bound on v_sam} \text{$\norm{\vv_t} \leq \norm{\nabla f_i(\x_t)} + \beta\rho$  at each step $t$.}
\end{align}
This holds since the $\beta$-Lipschitz gradient condition implies 
\begin{align}
\norm{\nabla f_i\left(\x_t + \rho \sigma_t \frac{\nabla f_i(\x_t)}{  \norm{\nabla f_i(\x_t) }}\right)} \leq \norm{\nabla f_i(\x_t)} +\beta \norm{\rho \sigma_t \frac{\nabla f_i(\x_t)}{  \norm{\nabla f_i(\x_t) }}}  =\norm{\nabla f_i(\x_t)} +\beta\rho  \,,
\end{align}
and the ``projecting-out'' operator $\proj^\perp_{\nabla \loss (\x_t) }$ only decreases the norm of the vector. Hence, it follows that  $\norm{\vv_t}  \leq \norm{\nabla f_i(\x_t)}+ \beta\rho$.

Now we show by induction that  
$f(\x_t)-f(\Phi(\x_t))\leq \frac{8\beta^3}{\alpha} \eta\rho^2$  holds during each step $t$. 
Suppose that it holds for $\x_t$ and consider $\x_{t+1}$.
Then from \autoref{lem:local}, it holds that 
$\norm{\x_t-\Phi(\x_t)}^2\leq \frac{16\beta^{3}}{\alpha^2} \eta\rho^2$, which implies that $\norm{\x_t-\Phi(\x_t)} = \oo{\eta^{1/2}\rho} = \ooo{\rho}$  as long as $\eps$ is sufficiently small.
Thus, from \eqref{eq:bound on v_sam}, it follows that  $\norm{\vv_t} \leq 2\beta \rho$, and hence, \autoref{lem:stay_near} implies that 
$f(\x_{t+1})-f(\Phi(\x_{t
+1}))\leq \frac{8\beta^3}{\alpha} \eta\rho^2$.

This conclusion together with \autoref{lem:local}  and the choice \eqref{exp:choice} imply the following conclusion:
\begin{align} \label{eq:dist to opt_sam}
\text{$\norm{\nabla f(\x_t)} \leq \eps_0~~\text{and}~~ \norm{\x_t -\Phi(\x_t)}\leq\oo{\eps_0}$ hold during each step $t$.}
\end{align} 
We now characterize the direction $\partial \Phi(\x_t)  \vv_t$.

\begin{lemma} \label{lem:trace_gradient_sam}
Let \autoref{as:local} hold  and each $f_i$ is four times coutinuously differentiable within the $\zeta$-neighborhood of $\Xstar$ and have $\beta$-Lipschitz gradients. 
Consider the parameter choice \eqref{exp:choice}. Then, for sufficiently small $\eps>0$, under the condition \eqref{eq:dist to opt_sam}, $\vv_t$ defined in \autoref{algo:SAM} (assume that it is well-defined as per \autoref{rmk:well_define}) satisfies
\begin{align} 
\E\partial \Phi(\x_t)   \vv_t  = \frac{1}{2}d\rho^2 \partial\Phi(\Phi(\x_t)) \nabla     \ttr(\Phi(\x_t)) + \oo{\constb\constc d\rho^2 \eps_0 }    +\oo{ \consta \rho^3} \,,
\end{align}
where $\constb \coloneqq \max_{i=1,\dots, n}\be{\partial(\nabla \mm_i/\norm{\nabla \mm_i})}$  and $\constc \coloneqq \be{\partial \Phi} \be{\nabla^3 f_i}$.
\end{lemma}
\begin{proof}
See \autoref{pf:lem:trace_gradient_sam}.
\end{proof} 

Notice an multiplicative factor of  $d$ appearing in the equation above, which shows an improvement over \autoref{lem:trace_gradient}. 
Using \autoref{lem:trace_gradient_sam}, we can prove the following formal statement of \autoref{lem:trace_descent_sam}.

\begin{lemma} \label{lem:trace_descent_formal_sam}
Let \autoref{as:local} hold and choose the parameters as per \eqref{exp:choice}.
Let $\eps>0$ be chosen sufficiently small and $\delta\in(0,1)$.
Then,  under the condition \eqref{eq:dist to opt_sam}, there exists an absolute constant $c>0$ s.t. the following holds:
if $\norm{\partial \Phi(\Phi(\x_t)) \nabla \ttr(\Phi(\x_t))}\geq \sqrt{\eps}$,  then
\begin{align}  
\E\ttr(\Phi(\x_{t+1})) - \ttr(\Phi(\x_t)) &\leq -  \frac{c}{\consta^3  \be{\partial \Phi }^2 \beta^2}d\nu^3\delta^3 \eps^3   \,.
\intertext{
On the other hand, if $\norm{\partial \Phi(\Phi(\x_t)) \nabla \ttr(\Phi(\x_t))}\leq \sqrt{\eps}$, then}   \E\ttr(\Phi(\x_{t+1})) - \ttr(\Phi(\x_t)) &\leq   \frac{c}{\consta^3  \be{\partial \Phi }^2 \beta^2}d\nu^3\delta^4 \eps^3   \,.
\end{align} 
\end{lemma} 
\begin{proof}
See \autoref{pf:lem:trace_descent_formal_sam}.
\end{proof}

Now the rest of the proof follows the probabilistic argument in \autoref{pf:thm:RS}.  
For $t=1,2,\dots, T$, let 
$A_t$ be the event where $\norm{ \partial \Phi( \Phi(\x_t)) \nabla \ttr( \Phi(\x_t)) } \geq \sqrt{\eps}$,
and let $R$ be a random variable equal to the ratio of desired flat minima visited among the iterates $\x_1,\dots,\x_T$. 
Let $P_t$ denote the probability of event $A_t$. 
Then, the probability of returning a $(\eps,\sqrt{\eps})$-flat minimum is simply equal to $\E  R  = \frac{1}{T}\sum_{t=1}^T (1- P_t)$.
Similarly to \autoref{pf:thm:RS}, using \autoref{lem:trace_descent_formal_sam}, we have  
\begin{align}   \E\ttr(\Phi(\x_{t+1})) - \E\ttr(\Phi(\x_t))  &\leq - \frac{c}{\consta^3  \be{\partial \Phi }^2 \beta^2}d\nu^3\delta^3\eps^3\cdot (P_t -  \delta (1-P_t)) \\
& = \frac{c}{\consta^3  \be{\partial \Phi }^2 \beta^2}d\nu^3\delta^3\eps^3\cdot\left\{ \delta  - (1+ \delta) P_t \right\}\,, 
\end{align}
which after taking sum over $t=0 \dots, T-1$ and rearranging yields
\begin{align}  
\frac{1}{T}\sum_{t=1}^T P_t \leq   \frac{\consta^3  \be{\partial \Phi}^2\beta^2 }{cT d\nu^3\delta^3\eps^3} +   \delta   \leq 2\delta\,.
\end{align}
Hence choosing 
\begin{align}
T= \tth{\frac{\consta^3  \be{\partial \Phi}^2\beta^2 }{d\nu^3\eps^3\delta^4}}\,,
\end{align}   $\E R$ is lower bounded by $1-\oo{\delta}$, which shows that
\autoref{thm:SAM}. 
This shows that $\xh$ is an $(\oo{\eps_0},\sqrt{\eps})$-flat minimum with probability at least $1-\oo{\delta}$.

Now we prove the refinement part. Let $\xh_0 \coloneqq \xh$.
Since $\nu \coloneqq   \min\{d,\eps^{-1/3}\} \leq \eps^{-1/3}$,
\begin{align}
\eps_0= \tth{\frac{\beta^{3/2}}{\alpha \consta^{3/2} \be{\partial \Phi } } \nu^{3/2}\delta^{3/2} \eps} \leq \oo{\sqrt{\eps}}
\end{align}
Hence, from \autoref{lem:local}, it then follows that $\norm{\nabla f(\xh_0)}\leq \oo{\sqrt{\eps}}$ and $f(\x_t) - f(\Phi(\x_t))\leq \oo{\eps}$.
Then, the linear convergence of GD under the PL inequality shows that GD with step size $\oo{\eps}$ finds an point $\xh_{T_0}$ s.t. $\norm{\xh_{T_0} - \Phi(\xh_{T_0})}\leq \eps/2$ in $T_0= \oo{\eps^{-1}\log (1/\eps)}$ steps.
On the other hand,  applying \autoref{lem:Phi_diff} with $\vv=\mathbf{0}$, it holds that 
\begin{align}
\norm{\Phi(\xh_{t+1})-\Phi(\xh_t)} = \oo{\eta^2 \norm{\nabla f(\xh_t)}^2}  =\oo{\eps^3}\,.
\end{align}
Therefore, it follows that
\begin{align}
\norm{\Phi(\xh_{T_0})-\Phi(\xh_0})  =\oo{\eps^3\cdot \eps^{-1} \log (1/\eps)} = \oo{\eps^{-2}\log (1/\eps)}\,.
\end{align}
Thus, we conclude that $\xh_{T_0}$ is  a  $(\eps,\sqrt{\eps})$-flat minimum.
This concludes the proof of \autoref{thm:SAM}.

\section{Proof of Auxiliary Lemmas}

\subsection{Proof of \autoref{lem:descent_gd}}
\label{pf:lem:descent}

Due to the $\beta$-gradient Lipschitz assumption, we have:
\begin{align*}
\loss(\x^+) & \le \loss(\x) + \inp{\nabla \loss(\x)}{  \x^+ -\x } + \frac{\beta}{2} \norm{ \x^+ - \x}^2 \\
&= \loss(\x)  - \eta\norm{\nabla \loss(\x)}^2 + \frac{\eta^2\beta}{2} \norm{\nabla \loss(\x) + \vv}^2\\
&\le \loss(\x) -\frac{1}{2}\eta (2- \eta\beta)\frac{1}{2\beta }\norm{\nabla \loss(\x)}^2 +\frac{\beta\eta^2}{2} \norm{\vv}^2\,.
\end{align*}
Hence, using the fact that $\eta\beta \leq 1$, which implies $-(2-\eta\beta) \leq -1$.

\subsection{Proof of \autoref{lem:local}}
\label{pf:lem:local}
To prove \autoref{lem:local}, it suffices to show the following:
\begin{align} \label{ineq:local}
\norm{\x-\Phi(\x)} 
\leq \sqrt{\frac{2(f(\x)-f(\Phi(\x)))}{\alpha}}\leq   \frac{1}{\alpha } \norm{\nabla f(\x)}  \leq \frac{\beta}{\alpha} \norm{\x-\Phi(\x)}\,.
\end{align} 
The proof essentially follows that of \citep[Lemma B.6]{arora2022understanding}.
We provide a proof below nevertheless to be self-contained.
Since  $\x$ is within $\zeta$-neighborhood of $\Xstar$, \autoref{as:local} implies that $\Phi$ is well-defined, and hence
letting $\x(t)$ be the iterate at time $t$ of a gradient flow starting at $\x$, we have
\begin{align}
\norm{\x-\Phi(\x)} = \norm{\int_{t=0}^\infty  \nabla f(\x(t)){\D t} }  \leq \int_{t=0}^\infty  \norm{\nabla f(\x(t))  } {\D t}\,.
\end{align}
Now due to the Polyak--\L{}ojasiewicz inequality, it holds that $\norm{\nabla f(\x_t)}^2 \geq 2\alpha  (f(\x)-f(\Phi(\x)))$.
Thus, we have 
\begin{align}
\int_{t=0}^\infty  \norm{\nabla f(\x(t))  } {\D t} &\leq \int_{t=0}^\infty  \frac{\norm{\nabla f(\x(t))}^2}{\sqrt{2\alpha  (f(\x)-f(\Phi(\x)))}} {\D t} \overset{(a)}{=} - \int_{t=0}^\infty  \frac{\frac{\D}{\D t}[f(\x(t)) -f(\Phi(\x))]}{\sqrt{2\alpha  (f(\x)-f(\Phi(\x)))}} {\D t}\\
&= -\sqrt{\frac{2}{\alpha}}\int_{t=0}^\infty  \frac{\D}{\D t}\sqrt{f(\x(t)) -f(\Phi(\x))}  {\D t} = \sqrt{\frac{2}{\alpha} (f(\x) -f(\Phi(\x)))}\,,
\end{align}
where $(a)$ follows from the fact
\begin{align}
\frac{\D}{\D t}[f(\x(t)) -f(\Phi(\x))] = \inp{\nabla f(\x(t))}{\dot \x (t)} = -\norm{\nabla f(\x(t))}^2\,.
\end{align}
Hence, we obtain
\begin{align}
\norm{\x-\Phi(\x)} \leq \sqrt{\frac{2}{\alpha} (f(\x) -f(\Phi(\x)))} \leq \frac{\norm{\nabla f(\x)}}{\alpha}  \,, 
\end{align}
where the last inequality is due to the PL condition.
Lastly, we have 
\begin{align}
\frac{1}{\alpha}\norm{\nabla f(\x)} = \frac{1}{\alpha}\norm{\nabla f(\x) - \nabla f(\Phi(\x))} \leq \frac{\beta}{\alpha}\norm{\x-\Phi(\x)}\,,
\end{align}
where the last inequality is due to $\beta$-Lipschitz gradients of $f$. This completes the proof.

\subsection{Proof of \autoref{lem:Phi_diff}}
\label{pf:lem:Phi_diff}

We first prove the first bullet point.
From the smoothness of $\Phi$, we obtain
\begin{align}
\Phi(\x^+) - \Phi(\x) &=\partial \Phi(\x) (-\eta (\nabla f(\x) +\vv )) + \oo{ \be{\partial\Phi} \norm{\x^+ - \x}^2}\\
&\overset{(a)}{=} - \eta \partial \Phi(\x)  \vv +\oo{ \be{\partial\Phi}\eta^2 (\norm{\nabla f(\x)}+ \norm{\vv}^2) } \,,
\end{align}
where in ($a$), we used the fact $\partial \Phi(\x) \nabla f(\x) =0$ from \autoref{lem:Phi}.
This, in particular, implies that 
\begin{align}
\norm{{\Phi(\x^+) - \Phi(\x)}}^2 &\leq  3 \be{\partial \Phi }^2 \eta^2\norm{\vv}^2 + 3\be{\partial\Phi}^2 \eta^4\norm{\nabla f(\x)}^4+ 3\be{\partial\Phi}^2\eta^4\norm{\vv}^4\\
&\leq  4 \be{\partial \Phi }^2 \eta^2\norm{\vv}^2 + 3\be{\partial\Phi}^2\eta^4\norm{\nabla f(\x)}^4\,,
\end{align}
as long as $\eta$ is sufficiently small since $ \eta^4 \norm{\vv}^4$ is a lower order term than $ \eta^2 \norm{\vv}^2$.

Next, we prove the second bullet point. From the smoothness of $f(\Phi)$, we have 
\begin{align}
&f(\Phi(\x^+)) - f(\Phi(\x)) = f(\Phi(\Phi(\x^+))) - f(\Phi(\Phi(\x)))\\
&\leq \inp{ \partial \Phi(\Phi(\x)) \nabla f(\Phi(\x))}{\Phi(\x^+) - \Phi(\x)} + \oo{(\be{\partial^2 \Phi} \be{\nabla f}+ \be{\partial \Phi} \be{\nabla^2 f})\norm{\Phi(\x^+) - \Phi(\x)}^2}\\  
& = \oo{\eta^2 (\be{\partial^2 \Phi} \be{\nabla f}+ \be{\partial \Phi} \be{\nabla^2 f})\be{\partial \Phi }^2 (\norm{\vv}^2 +\eta^2 \norm{\nabla f(\x)}^4)  }
\end{align}
where   ($a$)  used the fact $\partial \Phi(\Phi(\x)) \nabla f(\Phi(\x)) =0$ from \autoref{lem:Phi}.
And the same argument applies for $f(\Phi(\x)) - f(\Phi(\x^+))$, so we get the conclusion.

\subsection{Proof of \autoref{lem:stay_near}}
\label{pf:lem:stay_near}

By \autoref{lem:descent_gd}, we have 
\begin{align}
\loss(\x^+) \le \loss(\x) -\frac{1}{2}\eta   \norm{\nabla \loss(\x)}^2 +\frac{\beta\eta^2}{2}\norm{\vv}^2\,.    
\end{align}
Now we consider two different cases:
\begin{enumerate}
\item First, if $\norm{\nabla \loss(\x)}^2 \leq  2\beta\eta  \norm{\vv}^2$, then \autoref{lem:local} implies that 
\begin{align}
f(\x) -f(\Phi(\x)) \leq \frac{1}{2\alpha} \norm{\nabla f(\x)}^2 \leq \frac{\beta}{\alpha} \eta \norm{\vv}^2\,.
\end{align}
Hence, it follows that 
\begin{align}
f(\x^+) - f(\Phi(\x^+)) &\leq  f(\x) - f(\Phi(\x^+)) + \frac{\beta\eta^2}{2}\norm{\vv}^2\\
&\leq f(\x) - f(\Phi(\x)) + \frac{\beta\eta^2}{2}\norm{\vv}^2 + \oo{\eta^2\norm{\vv}^2}\\
&\leq \frac{\beta}{\alpha} \eta \norm{\vv}^2 + \oo{\eta^2 \norm{\vv}^2} \leq  \frac{2\beta}{\alpha} \eta \norm{\vv}^2\,,
\end{align}
as long as $\eta$ is sufficiently small.

\item  On the other hand if $ \norm{\nabla \loss(\x)}^2 \geq  2\beta\eta  \norm{\vv}^2$, then we have $f(\x^+) -f(\x) \leq - \frac{1}{2}\beta \eta \norm{\vv}^2$.
Next, from \autoref{lem:Phi_diff}, it holds that 
\begin{align}
|f(\Phi(\x^+)) - f(\Phi(\x))| =  \oo{\eta^2 \norm{\vv}^2}\,,
\end{align}
as $ \eta^4 \norm{\nabla f(\x)}^2 =\oo{\eta^5 \norm{\vv}^2}$ and $\eta^4\norm{\vv}^4$ are both lower order terms.
Thus, it follows that 
\begin{align}
f(\x^+) - f(\Phi(\x^+)) &\leq 
f(\x) - f(\Phi(\x)) - \frac{1}{2}\beta\eta\norm{\vv}^2 + \oo{\eta^2\norm{\vv}^2} \\
&\leq  f(\x) - f(\Phi(\x)) -\frac{1}{4} \beta\eta\norm{\vv}^2\,,
\end{align}
as long as $\eta$ is sufficiently small.
\end{enumerate}
Combining these two cases, we get the desired conclusion.


\subsection{Proof of \autoref{lem:trace_gradient}}
\label{pf:lem:trace_gradient}

Note that  by Taylor expansion,  we have 
\begin{align}
\nabla f(\x_t + \rho \per_t) =  \nabla f(\x_t)  + \rho \nabla^2 f(\x_t) \per_t + \frac{1}{2}\rho^2 \nabla^3 f(\x_t) \left[\per_t,\per_t\right] + \oo{\frac{1}{6}\be{\nabla^4 f}\rho^3}
\end{align}
This implies that  
\begin{align}
\vv_t &= \proj^\perp_{\nabla \loss (\x_t)} \nabla \loss \left( \x_t +\rho \per_t \right)  =     \proj^\perp_{\nabla \loss (\x_t)} \left[ \nabla \loss \left( \x_t +\rho  \per_t \right) -\nabla f(\x_t) \right] \\
&=  \proj^\perp_{\nabla \loss (\x_t)} \left[\rho \nabla^2 \loss (\x_t) \per_t + \frac{1}{2}\rho^2 \nabla^3 f(\x_t) \left[\per_t,\per_t\right] +\oo{\frac{1}{6}\be{\nabla^4f} \rho^3} \right]  \,.
\end{align}
Now from \autoref{lem:Phi}, $\partial \Phi (\x) \nabla f(\x) =\mathbf{0}$ for any $\x$ in the $\zeta$-neighborhood of $\Xstar$, it follows
\begin{align}
\partial\Phi(\x_t) \E \vv_t &=   \frac{1}{2}\rho^2 \partial\Phi(\x_t) \E\nabla^3 f(\x_t) \left[\per_t,\per_t\right] + \oo{\frac{1}{6}\be{\partial \Phi}\be{\nabla^4f} \rho^3} \\
& \overset{(a)}{=}  \frac{1}{2}\rho^2 \partial\Phi(\x_t) \nabla  \E\tr\left(\nabla^2 f(\x_t) \per_t \per_t^\top\right) + \oo{\frac{1}{6}\consta\rho^3} \\
& \overset{(b)}{=}   \frac{1}{2}\rho^2 \partial\Phi(\x_t) \nabla  \tr\left(\nabla^2 f(\x_t)  \frac{1}{d}\mathbf{I}_d \right) + \oo{\frac{1}{6}\consta \rho^3} \\
&=\frac{1}{2}\rho^2 \partial\Phi(\x_t) \nabla  \ttr\left(\x_t  \right) + \oo{\frac{1}{6}\consta \rho^3}
\end{align}
where $(a)$ is due to the fact that  $\nabla^3 f(\x) \left[\per_t,\per_t \right] = \nabla (\nabla^2 f(\x) \left[\per_t,\per_t \right]) = \nabla  \tr\left(\nabla^2 f(\x) \per_t \per_t^\top\right)$ for any $\x$, and $(b)$ uses the fact that $\E[\per_t \per_t^\top] = \frac{1}{d} \mathbf{I}_d$. 
Now due to $\consta$-Lipschitzness of $\partial\Phi(\cdot) \nabla  \ttr\left(\cdot \right)$, we have
\begin{align}
\partial\Phi(\x_t) \E \vv_t &=   \frac{1}{2}\rho^2 \partial\Phi(\Phi(\x_t)) \nabla  \ttr\left(\Phi(\x_t) \right) + \oo{\frac{1}{2}\rho^2 \consta \norm{\x_t -\Phi(\x_t)}} + \oo{\frac{1}{6}\consta \rho^3}\\
&=   \frac{1}{2}\rho^2 \partial\Phi(\Phi(\x_t)) \nabla  \ttr\left(\Phi(\x_t) \right) + \oo{\frac{1}{3} \consta \rho^3 }\,,
\end{align}
where the last line is due to \eqref{eq:dist to opt}, which implies $\consta \rho^2  \norm{\x_t- \Phi(\x_t)}  = \ooo{\consta \rho^3}$ as $\norm{\x_t- \Phi(\x_t)} =\oo{\eps_0} =\ooo{\rho}$.
This completes the proof.

\subsection{Proof of \autoref{lem:trace_descent_formal}}
\label{pf:lem:trace_descent_formal}

First, note from \autoref{lem:Phi_diff} and \eqref{eq:dist to opt} that 
\begin{align}
\Phi(\x_{t+1}) - \Phi(\x_t) &=  - \eta \partial \Phi(\x_t)  \vv_t + \oo{\be{\partial\Phi}\beta^2\eta^2\rho^2 } \,,\\
\norm{{\Phi(\x_{t+1}) - \Phi(\x_t)}}^2 &\leq  6\be{\partial \Phi }^2  \beta^2
\eta^2\rho^2 \,.
\end{align} 

Throughout the proof, we will use the notation $\partr \coloneqq \partial \Phi(\Phi(\x_t)) \nabla \ttr(\Phi(\x_t))$.
Then from the $\consta$-smoothness of $\ttr(\Phi)$,  and the fact that $\Phi(\Phi(\x_t))=\Phi(\x_t)$ it follows that
\begin{align}  
\ttr(\Phi(\x_{t+1})) - \ttr(\Phi(\x_t)) &= \ttr(\Phi(\Phi(\x_{t+1}))) - \ttr(\Phi(\Phi(\x_t))) \\
&\leq \inp{ \partr}{\Phi(\x_{t+1}) -\Phi(\x_t)} + \frac{1}{2}\consta \norm{\Phi(\x_{t+1})-\Phi(\x_t)}^2\\
&\leq \inp{ \partr }{- \eta \partial \Phi(\x_t)  \vv_t } +3 \consta  \be{\partial \Phi }^2 \beta^2
\eta^2\rho^2   \,.
\end{align} 
Applying \autoref{lem:trace_gradient}, we obtain 
\begin{align}  
&\E\ttr(\Phi(\x_{t+1})) - \ttr(\Phi(\x_t)) \leq - \frac{\eta\rho^2}{2} \norm{\partr}^2  +    \frac{1}{3}\consta \norm{\partr}\eta \rho^3 +3\consta  \be{\partial \Phi }^2 \beta^2
\eta^2\rho^2    \,.
\end{align} 
Now for a constant $c>0$, consider the following  the parameter choice \eqref{exp:choice}:
\begin{align}  
{   \eta ={\frac{1}{c\consta  \be{\partial \Phi }^2 \beta^2}\delta\eps }, \quad\quad \rho={\frac{1}{c\consta} \delta \sqrt{\eps}}}\,.
\end{align}
From this choice, it follows that
\begin{align}
- \frac{\eta\rho^2}{2} \norm{\partr}^2 &= - {    \frac{1}{2c^3\consta^3 \be{\partial \Phi }^2 \beta^2 }\delta^3 \eps^{2} \norm{\partr}^2}  \\ 
\consta \norm{\partr}\eta \rho^3  &= {    \frac{1}{c^4\consta^3 \be{\partial \Phi }^2 \beta^2 }\delta^4 \eps^{2.5} \norm{\partr}}\,,\\
\consta  \be{\partial \Phi }^2 \beta^2
\eta^2\rho^2  &= {    \frac{1}{c^4\consta^3 \be{\partial \Phi }^2 \beta^2 }\delta^4 \eps^{3}}
\end{align}
Hence, by choosing the constant $c$ appropriately large, one can thus ensure that
\begin{align}  
&\E\ttr(\Phi(\x_{t+1})) - \ttr(\Phi(\x_t)) \leq    \frac{1}{2c^3\consta^3 \be{\partial \Phi }^2 \beta^2 }\delta^3 \eps^{2} \left( -\norm{\partr}^2 + \frac{1}{4}\delta\eps^{1/2} \norm{\partr} +\frac{1}{4}\delta\eps  \right) \,.
\end{align} 
This completes the proof of \autoref{lem:trace_descent_formal}.

\subsection{Proof of \autoref{lem:trace_gradient_sam}}
\label{pf:lem:trace_gradient_sam}
For simplicity, let $\per_{i,t} \coloneqq  \frac{\nabla f_i(\x_t)}{  \norm{\nabla f_i(\x_t) }}$.
Note that  by Taylor expansion,  we have 
\begin{align}
\nabla f_i(\x_t + \rho \per_{i,t}) =  \nabla f_i (\x_t)  + \rho \nabla^2 f_i(\x_t) \sigma_t\per_{i,t} + \frac{1}{2}\rho^2 \nabla^3 f_i(\x_t) \left[\per_{i,t},\per_{i,t}\right] + \oo{\be{\nabla^4 f_i}\rho^3}
\end{align}
Using the facts that  $\partial \Phi (\x_t) \nabla f(\x_t) =\mathbf{0}$ (\autoref{lem:Phi}), we have $\partial\Phi(\x_t) \proj^\perp_{\nabla \loss (\x_t)} =\partial\Phi(\x_t)$, so the above equation implies that  
\begin{align}
\partial\Phi(\x_t) \vv_t =  \partial\Phi(\x_t) \left[\nabla f_i (\x_t)  + \rho \nabla^2 f_i(\x_t) \sigma_t\per_{i,t} + \frac{1}{2}\rho^2 \nabla^3 f_i(\x_t) \left[\per_{i,t},\per_{i,t}\right] + \oo{\be{\nabla^4 f_i}\rho^3}\right]  \,.
\end{align}
Taking expectation on both sides, we have the first two terms above vanish because $\E\partial\Phi(\x_t)  \nabla f_i (\x_t)= \partial\Phi(\x_t)  \nabla f (\x_t) =\mathbf{0}$ and $\E[\sigma_t]=0$. 
Thus, using the $\consta$-Lipschitzness of $\partial\Phi(\cdot) \nabla \tr(\nabla^2 f_i(\cdot) \per\per^\top )$ for a unit vector $\per$, we obtain
\begin{align}
\partial\Phi(\x_t) \E \vv_t &=   \frac{1}{2}\rho^2 \partial\Phi(\x_t) \E\nabla^3 f_i(\x_t) \left[\per_{i,t},\per_{i,t}\right] + \oo{\be{\partial \Phi}\be{\nabla^4f_i} \rho^3} \\
&{=}   \frac{1}{2}\rho^2 \partial\Phi(\x_t) \nabla   \E\tr \left(\nabla^2 f_i(\x_t) \per_{i,t}\per_{i,t}^\top \right)   +\oo{\consta \rho^3}\\
&{=}   \frac{1}{2}\rho^2 \partial\Phi(\Phi(\x_t)) \nabla   \E\tr \left(\nabla^2 f_i(\Phi(\x_t)) \per_{i,t}\per_{i,t}^\top \right)   +\oo{\frac{1}{2}\consta \rho^2\norm{\x_t-\Phi(\x_t)
}}+\oo{\consta \rho^3}\\
&{=}   \frac{1}{2}\rho^2 \partial\Phi(\Phi(\x_t)) \nabla   \E\tr \left(\nabla^2 f_i(\Phi(\x_t)) \per_{i,t}\per_{i,t}^\top \right)   +\oo{ \consta \rho^3}\,,  \label{eq:last for sam}
\end{align} 
where the last line is due to \eqref{eq:dist to opt}, which implies $\consta \rho^2  \norm{\x_t- \Phi(\x_t)}  = \ooo{\consta \rho^3}$ as $\norm{\x_t- \Phi(\x_t)} =\oo{\eps_0} =\ooo{\rho}$.
As we discussed in \autoref{sec:pfsketch_sam}, now the punchline of the proof is that at a minimum $\x^\star \in \Xstar$, the Hessian is given as 
\begin{align}
\nabla^2 f(\x^\star) =   \frac{1}{n}  \sum_{i=1}^n  \left[\frac{\partial^2\ell(z,y_i)}{\partial ^2 z}\Big\vert_{z=\mm_i(\x^\star) } \nabla \mm_i(\x^\star) \nabla \mm_i(\x^\star)^\top  \right] \,.
\end{align}
Hence, using the notations   
\begin{align}
\pp_i(\x)\coloneqq \frac{\nabla \mm_i(\x)}{\norm{\nabla \mm_i(\x)}}\quad \text{and}\quad
\lambda_i(\x) = \frac{1}{n} \cdot \frac{\partial^2\ell(z,y_i)}{\partial ^2 z}\Big\vert_{z=\mm_i(\x) } \cdot \norm{\nabla \mm_i(\x)}^2\,,
\end{align}
one can write the Hessians at a minimum $\x^\star \in \Xstar$  as 
\begin{align}\label{exp:hessians}
\nabla^2 f(\x^\star) =     \sum_{i=1}^n \lambda_i(\x^\star) \pp_i(\x^\star)\pp_i(\x^\star)^\top \quad \text{and} \quad    \nabla^2 f_i(\x^\star) =     \textcolor{red}{n}\lambda_i(\x^\star)\pp_i(\x^\star)\pp_i(\x^\star)^\top, \forall i \,.
\end{align} 
In particular, it follows that 
\begin{align} \label{exp:tr_hessian_sam}
\tr(\nabla^2 f(\x^\star)) = \sum_{i=1}^n \tr\left(\lambda_i(\x^\star) \pp_i(\x^\star)\pp_i(\x^\star)^\top\right) = \sum_{i=1}^n \lambda_i(\x^\star)\,.
\end{align}
Note that since  $\nabla f_i(\x_t) = \frac{\partial\ell(z,y_i)}{\partial z}\vert_{z=\mm_i(\x_t) } \nabla \mm_i (\x_t)$, we have $ \per_{i,t} = \frac{\nabla f_i(\x_t)}{\norm{\nabla f_i(\x_t)}} = \frac{\nabla \mm_i(\x_t)}{\norm{\nabla \mm_i(\x_t)}} = \pp_i(\x_t)$.
Using this fact together with the above expressions for the Hessians \eqref{exp:hessians}, one can further manipulate the expression for $\partial\Phi(\x_t) \E \vv_t$ in \eqref{eq:last for sam} as follows:
\begin{align}
\partial\Phi(\x_t) \E \vv_t &=  \frac{1}{2}\rho^2 \partial\Phi(\Phi(\x_t)) \nabla   \E\tr \left(n\lambda_i(\Phi(\x_t))\pp_i(\Phi(\x_t))\pp_i(\Phi(\x_t))^\top \pp_i(\x_t)\pp_i(\x_t)^\top \right)   +\oo{ \consta \rho^3}\\
&\overset{(a)}{=}  \frac{1}{2}\rho^2 \partial\Phi(\Phi(\x_t)) \nabla   \E\left[ n\lambda_i(\Phi(\x_t))(1+ \constb\norm{\x_t-\Phi(\x_t)})^2  \right]   +\oo{ \consta \rho^3}\\ 
&=  \frac{1}{2}\rho^2 \partial\Phi(\Phi(\x_t)) \nabla    \left[ \sum_{i=1}^n\lambda_i(\Phi(\x_t))(1+ \constb\norm{\x_t-\Phi(\x_t)})^2 \right]   +\oo{ \consta \rho^3}\\
&=  \frac{1}{2}d\rho^2 \partial\Phi(\Phi(\x_t)) \nabla    \left[ \frac{1}{d}\sum_{i=1}^n\lambda_i(\Phi(\x_t))(1+ \constb\norm{\x_t-\Phi(\x_t)})^2 \right]   +\oo{ \consta \rho^3}\\
&\overset{(b)}{=}  \frac{1}{2}d\rho^2 \partial\Phi(\Phi(\x_t)) \nabla     \ttr(\Phi(\x_t)) + \oo{\constb\constc d\rho^2 \norm{\x_t-\Phi(\x_t)}}    +\oo{ \consta \rho^3} \,,
\end{align} 
where in $(a)$, we use the fact   $\pp_i(\x_t) = \pp_i(\Phi(\x_t)) + \oo{\constb \norm{\x_t - \Phi(\x_t)}}$, and $\pp_i(\Phi(\x_t))$ is well-defined since we assumed that $\nabla \mm_i(\x)\neq \mathbf{0}$ for $\x\in \Xstar$, $\forall i=1,\dots,n$, and $(b)$ is due to \eqref{exp:tr_hessian_sam}.  This completes the proof since $\norm{\x_t-\Phi(\x_t)} = \oo{\eps_0}$ from the condition \eqref{eq:dist to opt_sam}.

\subsection{Proof of \autoref{lem:trace_descent_formal_sam}} 
\label{pf:lem:trace_descent_formal_sam}

Throughout the proof, we will use the notation $\partr \coloneqq \partial \Phi(\Phi(\x_t)) \nabla \ttr(\Phi(\x_t))$.
Similarly to \autoref{pf:lem:trace_descent_formal}, we have 
\begin{align}  
\ttr(\Phi(\x_{t+1})) - \ttr(\Phi(\x_t))  \leq \inp{ \partr }{- \eta \partial \Phi(\x_t)  \vv_t } +\oo{ \consta  \be{\partial \Phi }^2 \beta^2
\eta^2\rho^2 }  \,.
\end{align}
Applying \autoref{lem:trace_gradient_sam}, we then obtain 
\begin{align}  
&\E\ttr(\Phi(\x_{t+1})) - \ttr(\Phi(\x_t)) \leq - \frac{d\eta\rho^2}{2} \norm{\partr}^2  +  \oo{ \constb\constc d\eta\rho^2 \eps_0 \norm{\partr} +\consta \norm{\partr}\eta \rho^3 +\consta  \be{\partial \Phi }^2 \beta^2
\eta^2\rho^2  } \,.
\end{align} 
Now for a constant $c>0$, consider the following  the parameter choice \eqref{exp:choice_sam}:
\begin{align}  
\eta ={\frac{1}{c\consta  \be{\partial \Phi }^2 \beta^2}\nu\delta\eps  , \quad\quad \rho={\frac{1}{c\consta} \nu\delta \sqrt{\eps}}}\,,\quad\quad \eps_0=  \frac{\beta^{1.5}}{c^3\alpha \consta^{1.5} \be{\partial \Phi } } \nu^{1.5}\delta^{1.5} \eps 
\end{align}
From this choice, together with the fact $\nu^{1.5} = \min\{d,\eps^{-1/3}\}^{1.5} \leq \eps^{-1/2}$, it follows that 
\begin{align}
- \frac{\eta\rho^2}{2} \norm{\partr}^2 &= - {    \frac{1}{2c^3\consta^3 \be{\partial \Phi }^2 \beta^2 }d \nu^3\delta^3  \eps^{2} \norm{\partr}^2}  \\
\constb\constc d\eta\rho^2 \eps_0 \norm{\partr} &  =\oo{\frac{1}{c^6}d\nu^{4.5} \delta^{4.5} \eps^3 \norm{\partr} } =\oo{\frac{1}{c^6}d\nu^{3} \delta^{4.5} \eps^{2.5} \norm{\partr} }\\
\consta \norm{\partr}\eta \rho^3  &= {    \frac{1}{c^4\consta^3 \be{\partial \Phi }^2 \beta^2 }\nu^4\delta^4 \eps^{2.5} \norm{\partr}}\,,\\
\consta  \be{\partial \Phi }^2 \beta^2
\eta^2\rho^2  &= {    \frac{1}{c^4\consta^3 \be{\partial \Phi }^2 \beta^2 } \nu^4\delta^4 \eps^{3}}
\end{align}
Hence,  using the fact that $\nu = \min\{d,\eps^{-1/3}\}\leq d$ and by choosing the constant $c$ appropriately large, one can thus ensure that
\begin{align}  
&\E\ttr(\Phi(\x_{t+1})) - \ttr(\Phi(\x_t)) \leq    \frac{1}{2c^3\consta^3 \be{\partial \Phi }^2 \beta^2 }d\nu^3\delta^3 \eps^{2} \left( - \norm{\partr}^2 + \frac{1}{4} \delta\eps^{0.5} \norm{\partr} +\frac{1}{4} \delta\eps  \right) \,.
\end{align} 
This completes the proof of \autoref{lem:trace_descent_formal_sam}.

\end{document}